%% file: main.tex
\documentclass{article} 
\usepackage{amsmath}
\usepackage{iclr2027_conference,times} 
\usepackage[hyphens]{url}
\usepackage{graphicx}
\usepackage{bm}
\usepackage{natbib}
\let\originaleqref\eqref
\input{math_commands.tex}

\let\eqref\originaleqref
\usepackage{amssymb}
\usepackage{amsthm}
\usepackage{booktabs}
\usepackage{enumitem} 
\usepackage{cancel}
\usepackage{soul}
\usepackage{tikz}
\usepackage{float}
\usetikzlibrary{arrows.meta,   shapes.misc}
\usetikzlibrary{positioning,   calc}
\usepackage{threeparttable}

\usepackage{caption}
\usepackage[hidelinks]{hyperref}
\usepackage{cleveref}
\usepackage[most]{tcolorbox}
\usepackage{afterpage}

\newtheorem{theorem}{Theorem}[section]
\newtheorem{proposition}[theorem]{Proposition}
\newtheorem{lemma}[theorem]{Lemma}
\newtheorem{corollary}[theorem]{Corollary}
\newtheorem{example}[theorem]{Example}
\newtheorem{definition}[theorem]{Definition}

\newtheorem{remark}[theorem]{Remark}

\newcommand{\vR}{\bm{R}}

\title{Limit Analysis for Symbolic Multi-step\\Reasoning Tasks with Information\\Propagation Rules Based on Transformers}
\input{authors.tex}

\date{}
\iclrfinalcopy
\begin{document}
\maketitle
\fancyhead{}
\renewcommand{\headrulewidth}{0pt}

\begin{abstract}
Transformers are able to perform reasoning tasks, however the intrinsic mechanism remains widely open. In this paper we propose a set of information propagation rules based on Transformers and utilize symbolic reasoning tasks to theoretically analyze the limit reasoning steps. 
We show that the limit number of reasoning steps is between $2^{L-1}-1$ and $\frac{3^{L-1}-1}{2}$ for a model with $L$ attention layers in a single-pass. In addition, we find that, for problems with a fixed number of reasoning steps, Transformers exhibit a preference for generating sequences that satisfy certain ordering relations. Ultimately, we construct a Transformer and prove that the limiting reasoning depth is achievable.
\end{abstract}

\section{Introduction}
The transformer architecture introduced by \citep{Att_is_need2017} has  demonstrated capabilities across a wide range of tasks \citep{liu2018generating,devlin2019bert,radford2019language,touvron2023llama,openai2023gpt}
These models have reached performance levels comparable to top human contestants in the International Mathematical Olympiad (IMO) \citep{trinh2024solving}. 
While using Chain-of-Thought (CoT) \citep{wei2022chain} substantially improve multi-step logical reasoning accuracy by prompting models to generate explicit intermediate reasoning steps, they often exhibit an over-thinking phenomenon that consumes excessive computational resources and increases response time. This inefficiency highlights a critical question: what is the intrinsic single-pass reasoning capacity of these models?

In this work, in order to more intuitively understand the mechanism of multi-round information transmission in Transformer models and the boundary of information transmission capacity, we use a symbolic dataset that can simulate a class of multi-step reasoning to conduct our study\citep{wang2025understandinglanguagemodelsolve}. 
We propose a set of information propagation rules that faithfully simulate how Transformers transmit information. We then rigorously analyze, under these propagation rules, the theoretical limit of the reasoning depth that can be achieved by an $L$-layer Transformer with a sufficiently large hidden dimension in a single forward pass. Our analysis shows that the achievable number of reasoning steps lies between $2^{L-1}-1$ and $\frac{3^{L-1}-1}{2}$. This result follows from two key structural properties of Transformers: (i) in each layer, tokens can execute reasoning operations in parallel; and (ii) each embedding vector in a hidden layer can encode information from multiple tokens in distinct sublinear subspaces. 

Finally, we validate our theoretical analysis with numerical experiments, showing that a trained 3-layer Transformer with a sufficiently large hidden dimension can attain the theoretically predicted maximum reasoning depth and for tasks requiring a fixed number of reasoning steps, Transformers tend to favor output sequences that obey specific ordering relations.

\section{Transformer and Reasoning Mechanism}

\begin{figure}[!t]
   \centering
    \includegraphics[width=\textwidth]{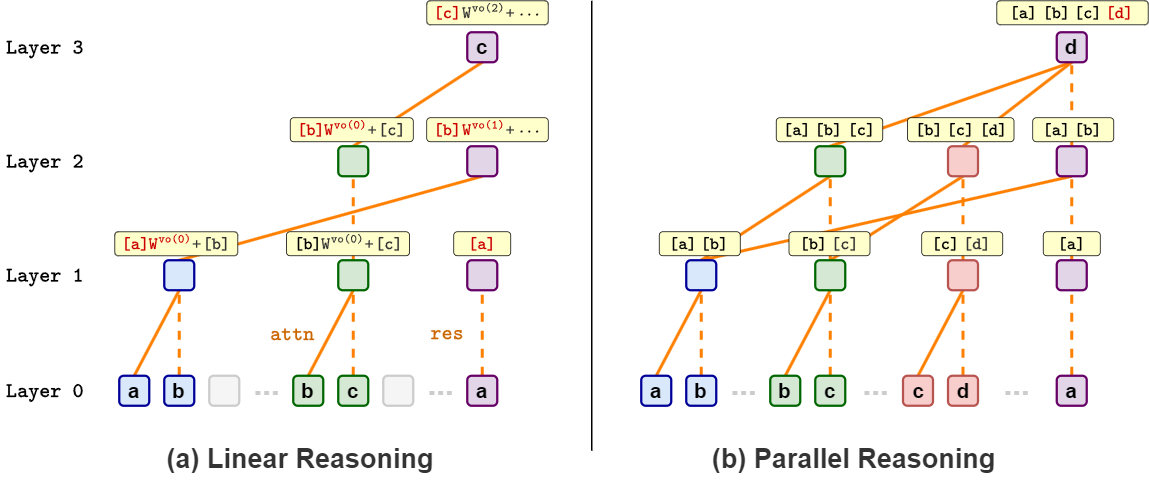} 
    \caption{Illustration of linear reasoning and parallel reasoning.}
    \label{fig:linear_and_parallel_reasoning}
\end{figure}

    \subsection{Transformer Architecture}

        We investigate a decoder-only Transformer with $L$-layer attention blocks. For integer $n$, given any sequence $(x_i)_{1\leqslant i \leqslant n}$, we denote its one-hot encoding
        \footnote{One-hot encoding is a technique that represents categorical data as binary vectors, where only one bit is set to 1 others are  set to 0.}
        as $\mX^{\mathrm{in}} \in \mathbb{R}^{n \times d}$ with $d$ as the dictionary size. 
        
        The model first applies an embedding layer including both token embedding and positional encoding to obtain the input representation as $\mX^{(0)} = {\mX^{\mathrm{emb}}} + \mX^{\mathrm{pos}} \in \mathbb{R}^{n \times d_m}$. Moreover, we denote the set of  word embeddings of each word in the dictionary as $W^{E}$.
        We shall use the single-head attention in each layer which is computed as follows:
        \begin{align*}
            \mathcal{A}^{(l)}(\mX^{(l)}) &= \operatorname{SoftMax}\left(\frac{\mathrm{mask}(\mX^{(l)}\mW^{q(l)}\mW^{k(l),  \mathsf{T}}\mX^{(l),  \mathsf{T}})}{\sqrt{d_k}}\right),  \\
            \quad  \mX^{\mathrm{qkv}(l)} &= \mathcal{A}^{(l)}({\mX}^{(l)}) {\mX}^{(l)} \mW^{v(l)}\mW^{o(l)}, 
        \end{align*}
        where $0\leqslant l\leqslant L$ and $\tilde{\sigma}$ is the softmax operator. For simplicity of expression,   we will abbreviate $\mW^{q(l)}\mW^{k(l),  \mathsf{T}}$ as $\mW^{qk(l)}$ and $\mW^{v(l)}\mW^{o(l),  \mathsf{T}}$ as $\mW^{vo(l)}$ in the following text. 
        Also, we ignore the normalization coefficient $\sqrt{d_k}$ in later sections for notational simplicity.
        The output of the $(l+1)$-th layer is obtained as:
        \begin{equation*}
        \begin{aligned}
            \mX^{\mathrm{ao}(l)} &= \mX^{(l)} + \mX^{\mathrm{qkv}(l)}, \\
            \mX^{(l+1)} &= \operatorname{RMSNorm}
            \bigl(f^{(l)}(\mX^{\mathrm{ao}(l)}) + \mX^{\mathrm{ao}(l)}\bigr).
        \end{aligned}
        \end{equation*}
        where $f^{(l)}(\cdot)$ represents the feedforward neural network of the $(l+1)$-th layer, with the same parameters at every position. Proposition~\ref{prop:shared_position_ffn} establishes this sharing for the construction. The final output (in the form of token indices within the vocabulary) is obtained as:
        \begin{equation*}
            \mY = \operatorname{SoftMax}({\mX}^{(L)}_n\mW^{p})\in \mathbb{R}^{d}.
        \end{equation*}
    
    \subsection{Induction Reasoning Mechanism}

    Based on numerous works on In-Context Learning, Induction Heads \citep{brown2020language,garg2022can,bietti2024birth, nichani2024transformers}, and recent studies on multi-step reasoning \citep{wang2025understandinglanguagemodelsolve, yu2025back}, the reasoning capability of Transformers can be largely attributed to a mechanism called the Buffer Mechanism for storing diverse information, together with adjacent position matching and same token matching for achieving information matching and transmission.

    \textbf{Buffer Mechanism}  The Buffer Mechanism is a crucial way for Transformers to store multiple pieces of information \citep{wang2025understandinglanguagemodelsolve}. Specifically, the interaction of information among tokens in a Transformer occurs in the attention module. Figure~\ref{fig:linear_and_parallel_reasoning}(a) illustrates the information flow of a 3-layer model performing 2-step reasoning, i.e., given a sentence of the form "...\texttt{[a]}\texttt{[b]}...\texttt{[b]}\texttt{[c]}...\texttt{[a]}" the model is required to output \texttt{[c]}. The dashed lines denote residual connections, while the solid lines denote information propagation induced by the attention mechanism. 
    When a token (e.g., $\texttt{[b]}$) attends to a previous token (e.g., $\texttt{[a]}$), 
    its next-layer state is not simply $\texttt{[a]+[b]}$, but rather $\texttt{[a]}\mW^{vo}+\texttt{[b]}$. 
    In other words, the Transformer stores the two pieces of information into subspaces spanned by different matrices through a linear transformation.

    \textbf{Adjacent Position Matching}  Similar to humans, language models rely heavily on the immediately preceding word when predicting the next word \citep{barbero2024round}. That is, the model can leverage positional encodings to establish connections between adjacent tokens. In fact, constructing such an attention weight matrix is not difficult. Assuming the positional encodings approximately satisfy $\vp_{i}^{\mathsf{T}}\vp_{i}=1, \vp_{i}^{\mathsf{T}}\vp_{j}=0, i\ne j$, it suffices to construct:
    \begin{equation}
        \mW^{qk}=\sum_{i=1}^{[l_\text{seq}/2]}\vp_{2i}\vp_{2i-1}^{\mathsf{T}},
    \end{equation}
    \begin{equation}
        (x_{2i}+\vp_{2i})\mW^{qk}(x_{2i-1}+\vp_{2i-1})^{\mathsf{T}}\approx 1,
    \end{equation}
    Clearly, by this method, we can construct attention between any adjacent tokens of fixed length. 
    However, due to the inherent diversity of language tasks, only the attention between the most adjacent pair is the most salient. We refer to this mechanism as adjacent position matching.

    \textbf{Same Token Matching} Same token matching is the most essential mechanism within induction heads. Its existence grants Transformers strong out-of-distribution generalization ability. As shown in the Figure~\ref{fig:linear_and_parallel_reasoning}(a), because both nodes in the first layer contain the same information $\texttt{[a]}$, they can attend to each other via the same token matching mechanism. Specifically, it suffices that the weight matrices satisfy $\mW^{qk(1)}\mW^{vo(0),\mathsf{T}}=I$, in which case
    \begin{equation}
    \begin{aligned}
        q(\texttt{[a]})k(\texttt{[a]}\mW^{vo(0)}+\texttt{[b]})^{\mathsf{T}}
        &\approx \texttt{[a]}\mW^{qk(1)}\mW^{vo(0),\mathsf{T}}\texttt{[a]}^{\mathsf{T}} \\
        &= \texttt{[a][a]}^{\mathsf{T}} \approx 1,
    \end{aligned}
    \end{equation}
    That is, the final token node will allocate nearly all of its attention to the previous node containing the same information $\texttt{[a]}$, thereby transmitting $\texttt{[b]}$ to the final node in the next layer. In this way, a single-step reasoning is achieved. Multi-step reasoning follows the same principle.

    \subsection{Parallel Reasoning}
    
    However, we note that the above-described mode, where each layer performs only one step of reasoning, is far from the upper limit of the Transformer model. As shown in Figure~\ref{fig:linear_and_parallel_reasoning}(b), adjacent position matching and same token matching can occur in parallel at multiple nodes within a layer, thereby enabling even shallow Transformer models to perform multi-step reasoning. We refer to this phenomenon as parallel reasoning. The central question considered in this paper is: given only adjacent position matching and same token matching, what are the upper and lower bounds of the parallel reasoning step that a transformer with $L$ layers attention blocks can perform?

\afterpage{%
\begin{figure}[!t]
    \centering
    \includegraphics[width=\textwidth]{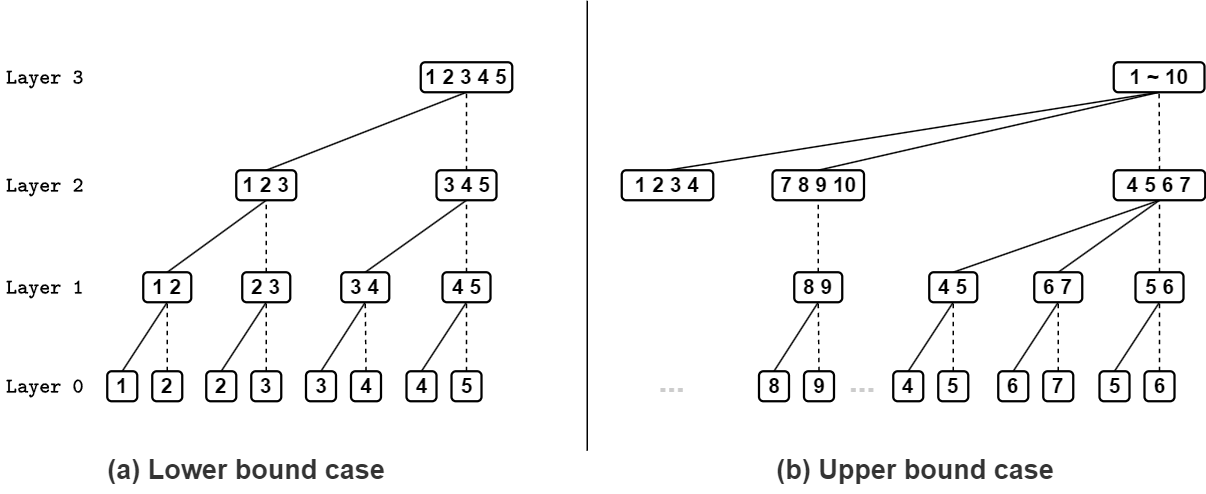} 
    \caption{Example of lower bound and upper bound of parallel reasoning.}
    \label{fig:example_lower_and_upper_bound}
\end{figure}
}
\section{Informal Theorems}
    To investigate the above question, we first consider the simplest case in which all reasoning relations are arranged sequentially. As shown in Figure~\ref{fig:example_lower_and_upper_bound}(a), the information flow of reasoning in this setting exhibits a clear "binary tree" structure. It then follows directly that the reasoning steps scale as $O(2^{L-1})$. In what follows, we will provide a rigorous proof of this result by mathematical induction. We note that permuting the order of reasoning pairs within the sequential arrangement does not disrupt the flow of sequential reasoning. Hence, when logical relations are arranged in sequence, the reasoning steps of a Transformer constitute the lower bound among all possible cases.

    On the other hand, we observe that for a 3-layer model, when the data are arranged as illustrated in Figure~\ref{fig:example_lower_and_upper_bound}(b), the final layer carries the maximum amount of information. The data in this case exhibit an evident fractal structure. The advantage of such a configuration is that each local terminal node can simultaneously match two preceding nodes by leveraging both the maximum and minimum information it carries, thereby expanding its information content. Consequently, the reasoning steps scale as $O(3^{L-1})$. Therefore, we arrive at the following informal conclusion:

\begin{theorem}[Informal Corollary \ref{Corollary_Transformer}]
     The maximal number of reasoning steps a transformer with $L$ layers attention blocks can perform has a lower bound $O(2^{L-1})$ and an  upper bound $O(3^{L-1})$.
\end{theorem}
Next, we will provide a formal statement of the problem and a rigorous proof of the conclusion.

\section{Symbolic Reasoning Task}\label{sec:symbolic_reasoning_task}
In this section,   we give a brief introduction to the reasoning task and some related definitions. Moreover,   we shall introduce the rules of information propagation. 

A reasoning task  typically involves a question and an answer to that question, along with the rule and process to get the answer.  For example, given  $A_1 \subseteq A_2$ and $A_2 \subseteq A_3$, the question is the relation of $A_1$ and $A_3$, and  the answer is  $A_1 \subseteq A_3$. We shall use a more symbolic way to express reasoning tasks as in the following example in figure \ref{Three_Step_Example}.

\begin{figure}[t]
  \centering

\begin{tikzpicture}[
    node distance=2cm,  
    box/.style={
        draw,   
        rounded rectangle,    
        minimum width=1.2cm,   
        minimum height=0.8cm,  
        thick,   
        font=\large,  
        align=center,
        scale = 0.8
    }
]
    \node[box] (1) {$1$};
    \node[box] (2) [right of=1] {$2$};
    \node[box] (3) [right of=2] {$3$};
    \node[box] (4) [right of=3] {4};
    \draw[->,   thick,   >=Stealth] (1) -- (2);
    \draw[->,   thick,   >=Stealth] (2) -- (3);
    \draw[->,   thick,   >=Stealth] (3) -- (4);
\end{tikzpicture}
\caption{Three steps reasoning task. One step reasoning leads to $2$,   two steps reasoning leads to $3$.}
\label{Three_Step_Example}
\end{figure}

We use a sequence $(1,2,2,3,3,4)$ to denote this reasoning task. Indeed, this sequence is composed of three bigrams $(1,2)$, $(2,3) $, $(3,4)$, and each bigram represents one step of reasoning. We call these bigrams reasoning pairs, which we will define below.
\begin{definition}
  A reasoning pair is an element in $\mathbb{Z}^2 $ of the form $\boldsymbol{a}_i = (\boldsymbol{a}^1_i,   \boldsymbol{a}^2_i)$ where  $i \in \mathbb{Z}$,   $\boldsymbol{a}^1_i,  \ \boldsymbol{a}^2_i \in \mathbb{Z}$ and $\boldsymbol{a}^1_i \neq \boldsymbol{a}^2_i $. 
\end{definition}

The set of all reasoning pairs is denoted as $\mathcal{A}$. $a^1_i \rightarrow a^2_i$ represents one step of reasoning.

It is natural that we shall define the $s$ step reasoning chain as follows.

\begin{definition}

An $s$ step reasoning chain is a finite sequence  $(\boldsymbol{a}_i)_{ 1 \leqslant i \leqslant s}$ with $\boldsymbol{a}_i \in \mathcal{A}$,   and it shall satisfy the following conditions:

\begin{itemize}
  \item $\boldsymbol{a}^2_i = \boldsymbol{a}^1_{i+1}$  for $1 \leqslant i \leqslant s-1$;
  \item For any subsequence  $(\boldsymbol{a}_{i_{m}})_{i_m \in \mathcal{I}}$ of $(\boldsymbol{a}_i)_{ 1 \leqslant i \leqslant s}$,   
  we have $\boldsymbol{a}^1_{\min\{ \mathcal{I} \}} \neq \boldsymbol{a}^2_{\max\{ \mathcal{I}\}}$ (no loop),   where $\mathcal{I}$ is a subset of $\{1,  2,  \cdots,  s\}$ containing at least two elements. 
\end{itemize}
\end{definition}
The first condition ensures that the reasoning chain does not break before the final step, and the second condition  ensures that there is no loop of arbitrary size in the reasoning task. 
For example, the sequence $((1,2),(2,3),(3,1))$ and $((1,2),(3,4),(4,5))$ are
not  reasoning chains.

Similarly, the subscript $i$
in this definition can take values from an infinite set like $\mathbb{Z}$, thereby defining an infinite chain of reasoning.
\begin{remark}
  For notational simplicity, here  and  in the sequel, we shall write $(\boldsymbol{a}_i)$ for  $(\boldsymbol{a}_i)_{i\in I}$ when the index set $I$ is clear from the context.
\end{remark}

In practice, a sentence may consist of reasoning pairs which are not in order.  Due to the mask condition which is common in the LLM, the order of reasoning pairs may influence the information propagation. To describe the order of these reasoning pairs and their relation to the reasoning chain, we define the elements of a symmetric group as permutations, and focus on the symmetric group $\operatorname{Sym}(\mathbb{Z})$.

\begin{definition}\
  Given a reasoning chain $(\boldsymbol{a}_m)_{m\in \mathbb{Z}}$ and a permutation $\sigma \in \operatorname{Sym}(\mathbb{Z})$, 
  a sequence $(x_i)_{i\in \mathbb{Z}}$ is called a reasoning sequence constructed from  $(\boldsymbol{a}_m)_{m\in \mathbb{Z}}$ and $\sigma $ if it satisfies:
\begin{equation}\label{xa_relation}
  x_i = \boldsymbol{a}_{\sigma (\lfloor \frac{i+1}{2} \rfloor) }^{2-(i \bmod 2)}.
\end{equation}
 Also, $(\boldsymbol{a}_m)_{m\in \mathbb{Z}}$ and $\sigma$ are called the constructing reasoning sequence and constructing permutation of $(x_i)$, respectively.
\end{definition}

When referring to a reasoning sequence $(x_i)$,   we are actually denoting the tuple $((x_i),  (\boldsymbol{a}_m),  \sigma)$. 
Moreover,   if $\sigma = \text{Id}$,   then the reasoning sequence is called a sorted reasoning sequence.
Note that from the relation \eqref{xa_relation} we also have 
\begin{equation}\label{ax_relation}
  \boldsymbol{a}_i = \boldsymbol{a}_{\sigma ( \sigma^{-1} (i))} = (x_{2 \sigma^{-1}(i)-1},  x_{2 \sigma^{-1}(i)} ),  
\end{equation}
where $\sigma^{-1}$ is the inverse of $\sigma$ satisfying $\sigma \circ \sigma^{-1} = \sigma^{-1} \circ \sigma = \text{Id} \in \text{Sym}(\mathbb{Z})$.

Likewise, in the case where the index ranges over a finite set, the reasoning sequence will be of finite length.
\vspace{0.6em}
\begin{figure}[tbp]
\centering
\begin{tikzpicture}[
    every node/.style={font=\footnotesize},  
    myarrow/.style={-stealth,   thick},  
    node distance=1cm 
 ]
\node (layer1) {\(\cdots,  a_1,   a_2,   a_3,   a_4,   \cdots\)};
\node [below=of layer1] (layer2) {\(\cdots,   \boldsymbol{a}_{\sigma(1)},   \boldsymbol{a}_{\sigma(2)},   \boldsymbol{a}_{\sigma(3)},   \boldsymbol{a}_{\sigma(4)},  \cdots \)}; 
\node [below=of layer2] (layer3) {\( \cdots,   x_1,   x_2,   x_3,   x_4,   x_5,   x_6,   x_7,   x_8,   \cdots\)};
\draw[myarrow] ($(layer1.south)+(0.15,  0)$) -- 
    node[pos=0.5,   right=1mm,   fill=white,   inner sep=1pt] {$\sigma$} 
    ($(layer2.north)+(0.15,  0)$);
\draw[myarrow] ($(layer2.north)+(-0.15,  0)$) -- 
    node[pos=0.5,   left=1mm,   fill=white,   inner sep=1pt] {$\sigma^{-1}$} 
    ($(layer1.south)+(-0.15,  0)$);
\draw[myarrow] ($(layer2.south)+(0.15,  0)$) --
node[pos=0.5,   right=1mm,   fill=white,   inner sep=1pt] {$\eqref{xa_relation}$} 
 ($(layer3.north)+(0.15,  0)$);

\draw[myarrow] ($(layer3.north)+(-0.15,  0)$) -- 
    node[pos=0.5,   left=1mm,   fill=white,   inner sep=1pt] {$\eqref{ax_relation}$} 
    ($(layer2.south)+(-0.15,  0)$);
\end{tikzpicture}
\caption{Relationship between reasoning chain and reasoning sequence.}
\end{figure}
Examples of sorted and permuted reasoning sequences are provided in Appendix~\ref{app:reasoning_examples}.

When considering a finite step reasoning sequence, the concept of reasoning start, which indicates where the reasoning task should begin, is also needed.
More specifically,   we consider an $s$ step reasoning sequence $(x_i)_{1\leqslant i \leqslant 2s}$ with constructing reasoning chain $(\va_m)_{1\leqslant m \leqslant s}$ and constructing permutation $\sigma$. Then one more element $x_{2s+1}$ is added to the end of the reasoning sequence $(x_i)$, and $x_{2s+1} = \va^1_{m_0}$ for some $\va_{m_0} \in (\va_m)$.

An example with a reasoning start is provided in Appendix~\ref{app:reasoning_examples}.

  Next   we define nodes which serves as containers of information. 
  \begin{definition}

  A node corresponding to a reasoning sequence $(x_i)$ is a set of two sets.
  More specifically,   for $i \in \mathbb{N}$,  
  an $l$th layer node is defined as $N^{l}_{i}=\{ V^l_i,     I^l_i \}$ 
  where $V^l_i$ is called  a value set whose elements are integers,  and $ I^l_i \subseteq \mathbb{Z} $ is an index set. Also,   we require that $V^l_i = \bigcup\limits_{i_{\alpha} \in I^{l}_i} \{x_{i_\alpha}\}$.
    \end{definition}
  We define the information quantity of a node $N^{l}_{i}=\{ V^l_i,    I^{l}_i \}$  as $C^{l}_i = | V^l_i |$.
  Moreover,   we denote $\mathcal{N}^l$ as the set of all $l$th layer nodes.
  
  We  denote the information propagation between two nodes  as $N^{l+1}_i =N^{l}_m\star N^{l}_i$ where $N^{l}_m\star N^{l}_i := \{ V^{l}_m \cup   V^{l+1}_i,     I^{l}_m \cup I^{l+1}_i \}$. The case  $i \neq m$ represents the attention mechanism, more specifically, the node $N^l_m$ attends to the node $N^l_i$ and the result is stored in the node $N^{l+1}_i$. The case $i = m$ represents the residual connection, in which case $N^{l+1}_i  = N^l_{i}$.
  Moreover,   there may be more than one node transmitting information to a node $ N^{l}_i$. In this case,   we denote the set of all such nodes as $N^l_{\mathcal{I}} \subseteq \mathcal{N}^{l}$ where $\mathcal{I}$ is some index set. The information propagation process in this case is then defined by $N^{l+1}_i = N^l_{\mathcal{I}}\star N^l_{i} := \{ \bigcup_{{i_\alpha \in \mathcal{I}}} V^{l}_{i_{\alpha}} \cup V^{l+1}_{i},        \bigcup_{{i_\alpha \in \mathcal{I}}} I^{l}_{i_{\alpha}} \cup I^{l+1}_{i}\} $.  

\section{Rules of information propagation}\label{Rules_of_Inf_Prop}
Our limit analysis allows \emph{value matching, same-position matching, and odd--even position matching in every attention block}, including block $0$. Odd--even matching uses the carried index sets: a sender containing $2r-1$ can match a receiver containing $2r$. The causal mask and residual retention apply throughout. Each block takes the union of all eligible messages; since adding information cannot remove a match, this full process includes every propagation obtained by selecting fewer transmissions.

For the reasoning chains defined above, Appendix~\ref{app:matching_order} proves that the full process and the simplified schedule---\emph{odd--even pairing in block $0$, followed by value matching}---give the last token $x_{2s+1}$ exactly the same set of successor values at every depth. Thus the information available for forward reasoning, and the maximal forward reasoning depth, are unchanged. We use the simplified rules below in the subsequent proofs. Total value sets and their sizes $C_i^l$ can differ between the two processes.

Here $N_i^l$ denotes the state after $l$ attention blocks, so block $0$ maps $N_i^0$ to $N_i^1$.
\begin{itemize}[leftmargin=*] \label{Rules}
  \item Rule 0 (Initial setup): The nodes in $0$th layer  are constructed as $N^{0}_i = \{ \{x_i\},    \{i\} \}$. For $l \geqslant 1$,   the $l$th layer of nodes are initially constructed as $N^{l}_i = \{ \emptyset,    \emptyset \}$.
  \item Rule 1 (Mask Condition): Information can be transmitted from node $j$ to node $i$ only when $j<i$. Thus every sender $j\in\mathcal I$ must precede the receiver; the receiver's own information is retained by Rule~4.

  \item Rule 2 (Odd--even pairing): For $l=1$, the information in an odd position node can be transmitted to the subsequent even position node. The mask condition is satisfied automatically. This contributes the information of $N^0_{2i-1}$ to $N^1_{2i}$.
  
  \item Rule 3 (Value matching): For $l \geqslant 2$,   a node $N^{l}_i$ is updated as $N^{l}_i = N^{l-1}_{\mathcal{I}}\star N^{l-1}_{i}  $ 
  provided there exists a set  $N^{l-1}_{\mathcal{I}} \subseteq \mathcal{N}^{l-1}$ satisfying the mask condition  $i_{\alpha} < i$ and    $V^{l-1}_{i_{\alpha}}  \cap V^{l-1}_{i} \neq \emptyset$  for all $ i_{\alpha} \in \mathcal{I}$.
  For the limit analysis, $\mathcal I$ contains all predecessors satisfying the condition.
  \item Rule 4 (Residual Connection): For $ l \geqslant 1$,   $\forall N^l_i \in \mathcal{N}^l$,   $N^l_i = \{ V^{l-1}_i \cup V^{l}_i,   I^{l-1}_i \cup I^{l}_i \}$.
\end{itemize}

\begin{remark}
  After all transmissions and the residual union, we continue to write $V_i^l$ and $I_i^l$ for the complete value and index sets at the output of layer $l$.
\end{remark}

\begin{remark}
  Matches use the preceding layer's states and transmit in parallel. Taking the residual union before or after these transmissions gives the same next-layer state.
\end{remark}
\section{Main Theorems} \label{Main_theorem}
The information-quantity theorems below use the simplified rules. Appendix~\ref{app:matching_order} preserves the last token's successor information, so the resulting forward reasoning bounds also hold with all three matching mechanisms in every block.

\begin{theorem}\label{Information_quantity_theorem}
    Under the rules of information propagation,   given any reasoning sequence $(( x_i)_{i\in \mathbb{Z}},(\boldsymbol{a}_m)_{m \in \mathbb{Z}},\sigma)$,   
     for  any  given $x \in \{x_i\}_{i \in \mathbb{Z}}$,   
    and for any $l \in \mathbb{Z}^+$ there exists $ i \in \mathbb{Z}^+$ such that $x \in V^l_i$,   and we have the following bound for $T^{l}(x) =\max\limits_{ i \in \mathbb{Z}}\{C^{l}_{i} \mid x \in V^{l}_{i}\}$:
    \begin{equation}
        \begin{aligned}
            &2^{l-1}+1 \leqslant T^{l}(x) \leqslant 3^{l-1}+1.
        \end{aligned}
    \end{equation}
\end{theorem}

 The whole proof is based on mathematical induction. Here we only give a sketch of the proof. The complete proof can be found in the Appendix~\ref{Proof_of_Main1}.

 Given a reasoning sequence $((x_i),(a_m),\sigma)$, when considering the lower bound, 
by Rule 2 the value sets of nodes in  layer $1$ contain only one reasoning pair except the case where only residual connection happens. 
Suppose that $j<i$ and the node $N^1_i$ contains ${\va^1_{m_i},\va^2_{m_i}}$ as value set, or simply  we say $N^1_i$ contains $\va_{m_i}$, and $N^1_j$ contains $\va_{m_j}$. 
Two cases may happen, $n_i +1 = n_j$ and $\va^2_{m_i} = \va^1_{m_j} $ or $n_j+1 = n_i$  and $\va^2_{m_j} = \va^1_{m_i} $. 
Both cases will lead to $N^2_{i}$ containing $\va_{m_i}$ and $\va_{m_j}$ by Rule 3. This process is the same for any other two nodes containing two adjacent reasoning pairs respectively. The process for layer $3$ is analogous to that for layer $2$: information contained in two nodes in layer $2$ is propagated to one node in layer $3$.
The whole structure is in fact a binary tree and hence the bound is powers of 2.

Regarding the upper bound, due to the mask condition (Rule 1), the permutation $\sigma$ may affect the information propagation. However, the upper bound is always bounded by the case when the mask condition is lifted. Therefore, we ignore the mask condition to find the upper bound.
Just like the proof of the lower bound which use two adjacent reasoning pairs, now we use three adjacent reasoning pairs. Suppose the nodes $N^1_i$, $N^1_j$ and $N^1_k$ contain the reasoning pairs $\va_{m_i}$,  $\va_{m_j}$ and $\va_{m_k}$ respectively and $m_i+1 = m_j = m_k-1$. 
Then by Rule 3, at least one of the node $N^2_{j}$ contains $\va_{m_i}$,  $\va_{m_j}$ and $\va_{m_k}$. As for layer $3$, there are nodes that contain information propagated from three nodes like $N^3_{j}$. The whole structure is a ternary tree and hence the bound is powers of 3.

\begin{theorem} \label{finite_sequence_theorem}
  Under the rules of information propagation \eqref{Rules},   for any $s$ step reasoning sequence $((x_i)_{1\leqslant i \leqslant 2s},(\boldsymbol{a}_m)_{1\leqslant m \leqslant s},\sigma)$ with reasoning start $x_{2s+1}$, for any $l$ satisfying that $1\leqslant l \leqslant 1+ \log_2 s$, we have the following estimate for $C^{l}_{2s+1}$.
  \begin{equation}
    2^{l-1} \leqslant C^{l}_{2s+1} \leqslant  3^{l-1}.
  \end{equation}
\end{theorem}
The proof of this theorem is similar to the proof of Theorem \ref{Information_quantity_theorem} with two main differences. 
First, we need to take into account the reasoning start which requires using  mathematical induction twice: once on the node of reasoning start and once on the nodes in the sequence. Second, the reasoning sequence is now finite, and we  require it to be long enough to support the structure of the binary tree and the ternary tree. The complete proof is included in the Appendix~\ref{Proof_of_Main2}.

\begin{remark}
    In the proof of this theorem,  the mask condition was relaxed to obtain the upper bound. However, we showed in Appendix~\ref{Examples} that there is a way to construct a large class of reasoning sequences such that the upper bound is attained 
    under mask condition. This proves that the upper bound is indeed tight.
\end{remark}

\begin{corollary}\label{Corollary_Transformer}
 Under the information propagation model, a Transformer with $L$ attention blocks and input length $n=2s+1$, where $s\in\mathbb{Z}^{+}$ and $1\leqslant L\leqslant1+\log_2 s$, has the following bounds on its maximal number of forward reasoning steps $S_p$:
 \begin{equation}
     2^{L-1}-1 \leqslant S_p \leqslant \frac{3^{L-1}-1}{2}.
 \end{equation}
\end{corollary}
\begin{remark}
    Information can propagate both backward and forward along the chain. For example, the sequence $(0,1,1,2,1)$ lets query $1$ collect $\{0,1,2\}$ in two blocks, but only $1\rightarrow2$ is a forward reasoning step. Appendix~\ref{app:matching_order} preserves the query's successor value set, which makes the bounds on $S_p$ valid for the full process as well.
\end{remark}

\section{Experiments and Discussions} 

\subsection{Training task}

To verify whether the Transformer model can achieve the expected reasoning capability, we conducted training specifically on Transformer architectures. As the number of layers increases, both the required amount of reasoning and the model capacity grow exponentially, leading to significant optimization difficulty. Therefore, we primarily focus on training a 3-layer Transformer model.

Our task aligns with the reasoning sequence described earlier. Specifically, we begin with a reasoning sequence denoted as $(x_{i})_{1\le i \le 2s}$. Subsequently, we introduce $x_{2s+1}$ as the starting point. The complete sequence $(x_{i})_{1\le i \le 2s+1}$ serves as the input to the transformer. The output corresponds to the reasoning result from the starting point with a fixed reasoning step (an example in Fig. \ref{fig:reason_exp}).

\begin{figure}[t]
   \centering
    \includegraphics[width=0.92\linewidth]{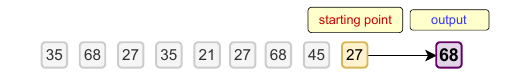} 
    \caption{A two-step reasoning example.}
    \label{fig:reason_exp}
\end{figure}
The detailed training settings are provided in the Appendix~\ref{appendix: experiment settings}. Reasoning pairs in the test set are totally different from the training set in order to prove transformer can learn reasoning instead of remembering these sequences. 

Firstly, this model is capable of solving 3-step reasoning problems with perfect accuracy as predicted by our theoretical analysis. We can see that such a model can be successfully trained in the Appendix~\ref{app:three_step_results}.

Then, We investigate scenarios where the model exhibits partial or complete failure. According to the theoretical analysis presented earlier Corollary \ref{Corollary_Transformer}, for reasoning step lengths in the interval [$2^{L-1}$, $\frac{3^{L-1}-1}{2}$], the model can produce correct answers under certain sequence ordering conditions. However, when the number of steps exceeds $\frac{3^{L-1}-1}{2}$, information propagation to the final token becomes insufficient, resulting in incorrect answers. Since this implies that high accuracy is unattainable in such regimes, experimental validation remains partial.

We consider a 3-layer Transformer. Given that $2^{3-1}=4$ and $\frac{3^{3-1}-1}{2}=4$, the model can solve 4-step reasoning tasks when sequence order conditions are satisfied, but fails for 5-step reasoning.
Experimental results under these settings are as follows:
\begin{itemize}
\item For 4-step reasoning, the model achieves a test accuracy of 46.1\%.
\item For 5-step reasoning, the test accuracy drops to 25.1\%.
\end{itemize}

The decrease in accuracy for the 5-step case provides empirical support for theoretical results. 

For clarity, we divide the 4-step reasoning data into two categories according to the sequence order structure predicted by our theory:
\begin{itemize}
    \item Canonical sequence: In the 4-step reasoning task, a sentence in which the 3rd fact appears after the 2nd and 4th facts.
    \item Non-canonical sequence: A sentence that does not satisfy the above condition.
\end{itemize}

According to our theoretical analysis, the model is expected to correctly resolve canonical sequences, whereas for non-canonical sequences it is unable to accumulate the necessary information and therefore effectively resorts to guessing. It can be seen that the Transformer exhibits a strong preference for canonical sequences, achieving approximately 90\% accuracy on canonical inputs but only around 20\% accuracy on non-canonical inputs, as shown in Figure~\ref{fig:canonical}. Additional detailed experimental verification can be seen in Appendix~\ref{app:four_step_results} and Appendix~\ref{app:causal_intervention}.

\begin{figure}[tbp] 
    \centering 
    \includegraphics[width=0.60\linewidth]{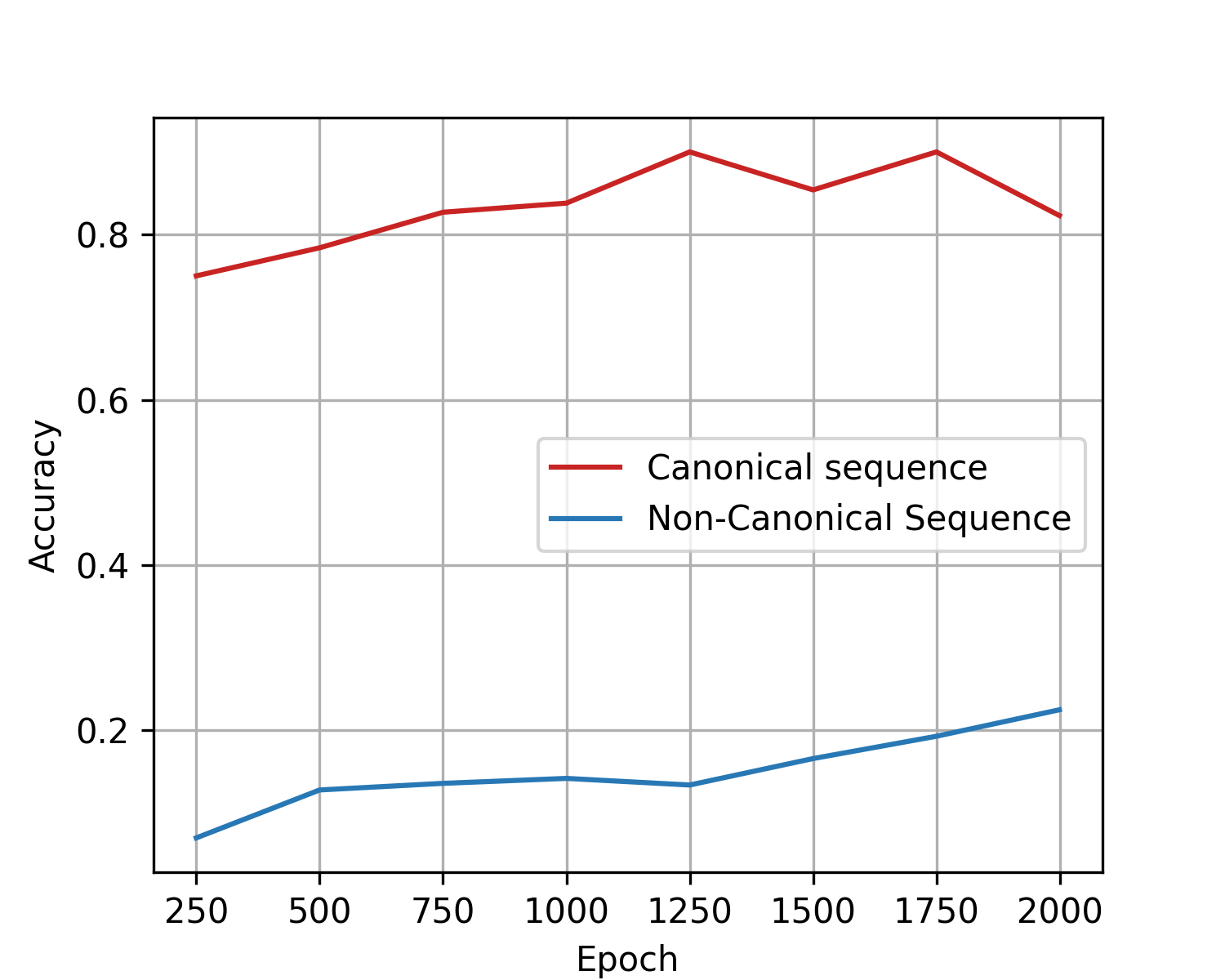}
    \caption{Test accuracy of sentences with different sequence relationships}
    \label{fig:canonical} 
\end{figure}

Complete training curves for 4-step reasoning tasks are provided in Appendix~\ref{app:four_step_results}.

To demonstrate the generality of our theory, we verify it on a 2-layer Transformer in Appendix~\ref{app:two_layer_results}.

\section{construction of the transformer}
To demonstrate that a Transformer can indeed perform reasoning at the limiting depth, we construct in the Appendix~\ref{Construction_of_Transformer} a model that stores the reasoning sequence using circulant matrices and is capable of achieving the aforementioned limiting reasoning depth.

\section{Related Work}

The ability of Transformers to perform complex reasoning is supported by theoretical work on their expressiveness. Prior research has shown that Transformers are Turing-complete in theory \citep{perez2019turing} and are universal approximators of sequence-to-sequence functions \citep{yun2019transformers}. However, there is a known gap between this theoretical capability and what Transformers can reliably learn from data \citep{wei2022statistically}, and their representational power has specific strengths and limitations \cite{sanford2023representational}.

To study multi-step reasoning in practice, various reasoning tasks have been proposed, such as
recognizing context-free grammars \citep{zhao2023transformersparsepredictingmasked}, 
matrix digits task \citep{webb2023emergent}, SET game tasks \citep{altabaa2023abstractors}, reasoning tasks designed by anchor function \citep{zhang2024anchor}.
\citet{kil2024ii, li2024making} use Chain-of-Thought (CoT) reasoning \citep{wei2022chain} to achieve multi-step reasoning via prompt engineering.
    \citet{zhang2023beam} introduced the
beam retrieval framework for multi-hop QA improving the few-shot QA performance of LLMs. 
    \citet{li2024understanding, yang2024large, shalev2024distributional} locate the  potential intermediate answers  within middle layers which play a causative role in shaping the final explicit reasoning results. \citet{wang2025understandinglanguagemodelsolve} use a symbolic dataset to study the mechanism of multi-step reasoning in Transformer models. Their work focuses on serial reasoning, where each layer performs one step of reasoning. As an extension of that work, we adopt the same dataset setting, but investigate the mechanism of parallel reasoning.
    \cite{zhang2024initialization,yao2025analysis} show that with small initialization, Transformers in condense regime \citep{luo2021phase,xu2025overview} can learn reasoning better.
    \citet{wang2025learning} investigated the $k$-fold task which is similar to the k-hop induction head task in \citet{sanford2024transformers}.

    \citet{sanford2024transformers} study the asymptotic depth complexity of the $k$-hop induction task. They construct a causally masked, single-head Transformer with constant width and depth $\lfloor\log_2 k\rfloor+2$, and establish its asymptotic optimality through an $\Omega(\log k)$ depth lower bound under an MPC conjecture and model-size constraints. Their pointer-doubling construction captures binary-tree information propagation, but the asymptotic result does not characterize which sequences remain solvable beyond the corresponding doubling range at a fixed depth. We derive rigorous upper and lower bounds on reasoning steps at a fixed depth $L$ under the proposed information propagation rules, and show that exceeding the binary-tree doubling range does not preclude correct reasoning: favorable sequence orderings allow tokens to aggregate additional information. These results offer a possible explanation for the non-negligible accuracy beyond the doubling range observed in their experiments, while remaining consistent with their asymptotic depth lower bound.
\label{main-text-end}

\clearpage
\subsection*{AI use statement} 
During the preparation of this manuscript, we used DeepSeek-V3 (July 2024) to assist with language polishing (grammar, wording). 
All substantive content, ideas, and core writing remain the authors' own, and the authors are fully responsible for any error including those introduced during polishing.
OpenAI Codex assisted with the ICLR 2027 LaTeX formatting, appendix integration.

\bibliographystyle{iclr2027_conference}
\bibliography{references}

\clearpage
\appendix
\section*{Appendix}
\input{appendix.tex}
\clearpage
\input{appendix_extra.tex}

\end{document}

%% file: math_commands.tex
\usepackage{amsmath,amsfonts,bm}

\def\eqref#1{equation~\ref{#1}}

\def\1{\bm{1}}

\def\va{{\bm{a}}}
\def\vb{{\bm{b}}}

\def\vp{{\bm{p}}}

\def\vr{{\bm{r}}}
\def\vs{{\bm{s}}}

\def\vv{{\bm{v}}}

\def\mA{{\bm{A}}}

\def\mI{{\bm{I}}}

\def\mM{{\bm{M}}}

\def\mQ{{\bm{Q}}}
\def\mR{{\bm{R}}}

\def\mV{{\bm{V}}}
\def\mW{{\bm{W}}}
\def\mX{{\bm{X}}}
\def\mY{{\bm{Y}}}

\DeclareMathAlphabet{\mathsfit}{\encodingdefault}{\sfdefault}{m}{sl}
\SetMathAlphabet{\mathsfit}{bold}{\encodingdefault}{\sfdefault}{bx}{n}



%% file: authors.tex
\author{\parbox{0.97\textwidth}{\centering
Qin Tian$^{2,\dagger}$, Yuhan Chen$^{3,\dagger}$, Zhiwei Wang$^{1,2,\dagger}$,\\
Zhi-Qin John Xu$^{1,2,*}$\\[5pt]
\normalfont\small
$^1$Institute of Natural Sciences, MOE-LSC, Shanghai Jiao Tong University\\
$^2$School of Mathematical Sciences, Shanghai Jiao Tong University\\
$^3$School of Mathematics and Statistics, Wuhan University\\[4pt]
$^\dagger$These authors contributed equally as co-first authors.\\
$^*$Corresponding author: \href{mailto:xuzhiqin@sjtu.edu.cn}{\texttt{xuzhiqin@sjtu.edu.cn}}
}}

%% file: appendix.tex
\section{Proof of Theorem \ref{Information_quantity_theorem}}\label{Proof_of_Main1}

We first prove the lower bound. 

According to Rule 0,   the nodes in first layer are constructed as $N^1_i = \{ (x_i),  \{ i\},  \{ i\}\}$.
According to Rule 2,   we have the nodes in second layer are constructed as $N^1_{2i} =\{ \{x_{2i-1},   x_{2i}\},  \{ 2i\},  \{ 2i-1,  2i\}  \}  $ for   $i \in \mathbb{Z}$. Clearly,   we have $C^1_{2i} = 2$.

We use mathematical induction to prove that for $k \geqslant 2$,   for any $x_j \in \{x_i\}_{i \in \mathbb{Z}}$,   there exists an integer $i\in \mathbb{Z}$ such that 
\begin{itemize} \label{Induction_hypothesis}
  \item $ x_j \in V^k_i$,  
  \item 
  We denote $\sigma({I}^k_i)$ to be the set $\{ \sigma(\lfloor \frac{i_m+1}{2} \rfloor)\}_{i_m \in {I}^k_i} $,   
  and we have 
   $\max{\sigma({I}^k_i)} - \sigma(\lfloor \frac{j+1}{2} \rfloor) \geqslant 2^{k-1}-1$ which implies that $C^k_i \geqslant 2^{k-1}+1$.
\end{itemize}

For $k=2$,   given any $x_j \in (x_i)$,   according to Rule 2,   we know that $x_j$ is contained in the value set $V^1_{2\lfloor \frac{j +1}{2} \rfloor} $ of node $N^1_{2\lfloor \frac{j +1}{2} \rfloor}$. 

For simplicity,    we assume that $j$ is an even number,   then $V^1_{2\lfloor \frac{j +1}{2} \rfloor} =  V^1_{j} = \{x_{j-1},   x_{j}\}$ and $I^1_{j} = \{j-1,  j\}$.
We now want to find a node $N^{i}_2$ such that $V^1_{i}  \cap V^1_{j}  = \{x_j\}$.
To do so,   by definition of reasoning sequence and using the  relation \eqref{xa_relation} we know that $x_{j-1}$ and $x_{j}$ comes from the reasoning pair $\boldsymbol{a}_{\sigma(\frac{j}{2})}$.
By definition of reasoning sequence and relation \eqref{ax_relation} we have 
\begin{equation}
  \begin{aligned}
    &\boldsymbol{a}^2_{\sigma(\frac{j}{2})} = \boldsymbol{a}^1_{\sigma(\frac{j}{2}) + 1},  \\
    &\boldsymbol{a}_{\sigma(\frac{j}{2}) + 1} =  \left(x_{2 \sigma^{-1}\left(\sigma(\frac{j}{2}) + 1 \right) -1},  x_{2 \sigma^{-1}\left(\sigma(\frac{j}{2})+1\right)}\right).
  \end{aligned}
\end{equation}
Set  $i = {2 \sigma^{-1}\left(\sigma(\frac{j}{2}) +1\right)} $. 
It is now clear that the node $N^{1}_{i}$ have value set $V^{1}_{i} $ satisfying $V^{1}_{i}  \cap V^{1}_{j}  = \{x_j\}$ and index set $I^{1}_i = \{i-1,  i \}$. 
By Rule 3,   the node $N^{2}_{\max\{ i,  j\}}$ have value set $V^{2}_{\max\{ i,  j\}}$ such that $\{ x_{j-1},  x_{j},  x_{i+1} \} \subseteq V^{2}_{\max\{ i,  j\}}   $. Hence,   $C^{2}_{\max\{ i,  j\}} \geqslant 3$ and 
$\sigma(\frac{j}{2}) + 1 - \sigma(\frac{j}{2}) = 1 \geqslant 2^{2-1}-1 $.

Now we assume the induction hypotheses hold for $3 \leqslant k \leqslant k_0$ and consider the case $k=k_0+1$. 
There exists a node $N^{k_0}_{m_0}$ such that 

\begin{align}
  &x_j \in V^{k_0}_{m_0},  \\
  &\max{\sigma({I}^{k_0}_{m_0})} - \sigma( \frac{j}{2} ) \geqslant 2^{k_0-1}-1,   \label{pari_counting1}\\
  &C^{k_0}_{m_0} \geqslant 2^{k_0-1}+1.
\end{align}
Set $ {\alpha}=\max\sigma(I^{k_0}_{m_0})$ then using relation \eqref{ax_relation} we know that $ \boldsymbol{a}^2_\alpha = x_{2 \sigma^{-1}(\alpha)} \in V^{k_0}_{m_0}$. 
Using definition of reasoning sequence and relation \eqref{xa_relation} we know that 
\begin{equation}
  x_{2\sigma^{-1}(\alpha+1)-1} = \boldsymbol{a}^1_{\alpha+1} = \boldsymbol{a}^2_{\alpha} = x_{2 \sigma^{-1}(\alpha)}.
\end{equation}
Note that $\sigma( \lfloor \frac{2\sigma^{-1}(\alpha+1)-1+1}{2} \rfloor) = \alpha+1$. And then the induction hypotheses implies that there is a node $N^{k_0}_{m_1}$ such that 
\begin{align}
  &x_{2\sigma^{-1}(\alpha+1)-1} \in V^{k_0}_{m_1},  \\
  &\max{\sigma({I}^{k_0}_{m_1})} - (\alpha +1 )\geqslant 2^{k_0-1}-1,   \label{pari_counting2}\\
  &C^{k_0}_{m_1} \geqslant 2^{k_0-1}+1.
\end{align}
This construction ensures that $x_{2\sigma^{-1}(\alpha+1)-1}=x_{2 \sigma^{-1}(\alpha)} \in V^{k_0}_{m_1} \cap V^{k_0}_{m_0}$,   
then by Rule 3,   we know that the node $N^{k_0 +1}_{\max\{m_0,  m_1\}}$ have value set $V^{k_0+1}_{\max\{m_0,  m_1\}}$ and index set $I^{k_0+1}_{\max\{m_0,  m_1\}}$ satisfying that 
\begin{align}
  &x_j \in V^{k_0+1}_{\max\{m_0,  m_1\}},   \\
  &I^{k_0}_{m_0} \cup I^{k_0}_{m_1} \subseteq I^{k_0+1}_{\max\{m_0,  m_1\}}. \label{index_union}
\end{align}
Combining \eqref{pari_counting1},   \eqref{pari_counting2} and \eqref{index_union} leads to 
\begin{equation*}
  \max{\sigma(I^{k_0+1}_{\max\{m_0,  m_1\}}) } - \sigma({\frac{j}{2}}) \geqslant 2^{k_0}-1,  
\end{equation*}
which implies that $C^{k_0+1}_{\max\{m_0,  m_1\}} \geqslant 2^{k_0+1 -1} +1$. 
This completes the proof for lower bound.

  Next we prove the upper bound. 
  
  It is clear that if Rule 1 is removed,   the information quantity in each node can only maintain unchanged or increase. Therefore,   we consider the no-mask condition,   without loss of generality, we consider the reasoning sequence to be $\{x_i = \lfloor \frac{i}{2}\rfloor \}_{i\in \mathbb{N}}$ with constructing permutation $\sigma= \text{Id}$ and constructing reasoning chain $(\boldsymbol{a}_m) =( (m-1,  m))$.

We will use mathematical induction to prove that 
for $k \geqslant 2$ and $m\in \mathbb{Z}$,   the value set of node $N^{k}_{2m}$ satisfies $V^{k}_{2m}\subseteq \{m-\frac{3^{k-1}+1}{2},  m-\frac{3^{k-1}+1}{2} +1,    \cdots,  m+\frac{3^{k-1}-1}{2}\}$ and $C^{k}_{2m} \leqslant 3^{k-1}+1$. 

Since the index set of a node will play no rule in this proof,   we will just ignore them in the expression of a node.

When $k= 2$,   By Rule 2 the nodes $N^1_i$ are of the form $N^1_{2i}=\{\{ x_{2i-1},   x_{2i}\},     \{2i\} \} $.
Then by Rule 3  the nodes in layer $3$ are of the form $N^2_{2i}
=  \{ \{ i-2,   i-1,   i,  i+1  \} \{2i\}\}$. Therefore,   the conclusion holds for $k=2$.

Assuming the conclusion holds for $k\leqslant k_{0}$.
When $k=k_0 + 1 $,   by inductive hypotheses,   there are three nodes $ N^{k_0}_{2m}$,   $ N^{k_0}_{2m_1}$,   $ N^{k_0}_{2m_2}$ with value sets $ V^{k_0}_{2m}$,   $ V^{k_0}_{2m_1}$,   $ V^{k_0}_{2m_2}$.
We require that  
\begin{equation}
  \begin{aligned}
    m+\frac{3^{{k_0}-1}-1}{2} = m_1-\frac{3^{{k_0}-1}+1}{2},  \\
    m_2+\frac{3^{{k_0}-1}-1}{2} = m-\frac{3^{{k_0}-1}+1}{2}.
  \end{aligned}
\end{equation}
Simple calculation shows that 
\begin{equation}
  m_1 = m+ 3^{{k_0}-1},  \ m_2=m-3^{{k_0}-1}.
\end{equation}
And the three nodes 
$ N^{k_0}_{2m}$,   $N^{k_0}_{2(m+ 3^{{k_0}-1})} $,   $ N^{k_0}_{2(m-3^{{k_0}-1})}$ have value sets
\begin{align}
V^{k_0}_{2m }\subseteq \{m-\frac{3^{{k_0}-1}+1}{2},  m-\frac{3^{{k_0}-1}+1}{2} +1,    \cdots,  m+\frac{3^{{k_0}-1}-1}{2}\},   \\
V^{k_0}_{2(m+ 3^{{k_0}-1})} \subseteq \{m+\frac{3^{{k_0}-1}-1}{2},  m+\frac{3^{{k_0}-1}-1}{2} +1,    \cdots,  m+\frac{3^{{k_0}-1}-1}{2}\},   \\
V^{k_0}_{2(m-3^{{k_0}-1})} \subseteq \{m-\frac{3^{{k_0}-1}+1}{2},  m-\frac{3^{{k_0}-1}+1}{2} +1,    \cdots,  m-\frac{3^{{k_0}-1}+1}{2}\}.
\end{align}
Again by Rule 3,   we know that the node  $N^{k_0+1}_{2m} = N^{k_0}_{2(m+ 3^{k-2})} \star N^{k_0}_{2(m-3^{k-2})} \star N^{k_0}_{2m}$ have value set $V^{k_0+1}_{2m} \subseteq \{m-\frac{3^{{k_0}}+1}{2},   m-\frac{3^{{k_0}}+1}{2}+ 1,  \cdots,  m+\frac{3^{{k_0}}-1}{2}\}$,   and therefore,   $C^{k_0+1}_{2m} \leqslant 3^{k_0} + 1$. This completes the proof.

\section{Proof of Theorem \ref{finite_sequence_theorem}} \label{Proof_of_Main2}

We shall also use the mathematical induction to prove this theorem. 

We first assert that  
for $1\leqslant l \leqslant 1+\log_2 s$ and for $j$ satisfying $1 \leqslant \sigma(\lfloor \frac{j+1}{2} \rfloor) \leqslant s- 2^{l-1} +1$ there exists a node $N^l_{k}$ such that 
\begin{equation} \label{finite_inductive_hypotheses}
\bigcup\limits_{ 0 \leqslant \alpha \leqslant 2^{l-1}-1} \{ x_{2\sigma^{-1}\left(\sigma(\lfloor \frac{j+1}{2} \rfloor)+\alpha\right)-1},  x_{2\sigma^{-1}\left(\sigma(\lfloor \frac{j+1}{2} \rfloor)+\alpha\right)}  \} \subseteq V^{l}_k.
\end{equation}

For $l=1$,   the relation \eqref{finite_inductive_hypotheses} can be verified easily since after position matching the nodes in $\mathcal{N}^1$ are of the form $N^1_{2m} = \{ \{ x_{2m-1},   x_{2m}\},    \{2m\},   \{2m-1,   2m\} \}$. 
And note that $x_j\in \{   x_{2\sigma^{-1}\left(\sigma(\lfloor \frac{j+1}{2} \rfloor)\right)-1},  x_{2\sigma^{-1}\left(\sigma(\lfloor \frac{j+1}{2} \rfloor)\right)} \} \subseteq V^{1}_{2 \lfloor \frac{j+1}{2} \rfloor}$.

Now we assume that \eqref{finite_inductive_hypotheses} holds for $l=l_0$,   then we prove it still holds for $l=l_{0}+1$.
Given $x_j \in (x_i)$ with $1 \leqslant \sigma(\lfloor \frac{j+1}{2} \rfloor) \leqslant s- 2^{l_0} +1$,   it is clear that $1 \leqslant \sigma(\lfloor \frac{j+1}{2} \rfloor) \leqslant s- 2^{l_0-1} +1$. Then by inductive hypotheses we know that there exists  a node $N^{l_0}_{k_1}$ such that 
\begin{equation} 
\bigcup\limits_{ 0 \leqslant \alpha \leqslant  2^{l_0-1}-1} \{ x_{2\sigma^{-1}\left(\sigma(\lfloor \frac{j+1}{2} \rfloor)+\alpha\right)-1},  x_{2\sigma^{-1}\left(\sigma(\lfloor \frac{j+1}{2} \rfloor)+\alpha \right)}  \} \subseteq V^{l_0}_{k_1}.
\end{equation}
On the other hand,   since $1 \leqslant \sigma(\lfloor \frac{j+1}{2} \rfloor) \leqslant s- 2^{l_0} +1 $ we have  $1 \leqslant \sigma(\lfloor \frac{j+1}{2} \rfloor)  + 2^{l_0-1}  \leqslant s- 2^{l_0-1} +1$,   and again by inductive hypotheses there exists a node $N^{l_0}_{k_2}$ such that 

\begin{equation} 
\bigcup\limits_{ 0 \leqslant \alpha \leqslant  2^{l_0-1}-1} \{ x_{2\sigma^{-1}\left(\sigma(\lfloor \frac{j+1}{2} \rfloor) + 2^{l_0-1}+\alpha\right)-1},  x_{2\sigma^{-1}\left(\sigma(\lfloor \frac{j+1}{2} \rfloor)+ 2^{l_0-1}+\alpha\right)}  \} \subseteq V^{l_0}_{k_2}.
\end{equation}

Note that by the definition of reasoning sequence we have
\begin{equation*}
  x_{2\sigma^{-1}\left(\sigma(\lfloor \frac{j+1}{2} \rfloor) + 2^{l_0-1}\right)-1} = x_{2\sigma^{-1}\left(\sigma(\lfloor \frac{j+1}{2} +\rfloor) + 2^{l_0-1}-1 \right)}.
\end{equation*}

By Rule 3,   there exists a node $N^{l_0+1}_{\max\{k_1,  k_2\}}$ such that 
\begin{equation*} 
\bigcup\limits_{ 0 \leqslant \alpha \leqslant  2^{l_0}-1} \{ x_{2\sigma^{-1}\left(\sigma(\lfloor \frac{j+1}{2} \rfloor)+\alpha  \right)-1},  x_{2\sigma^{-1}\left(\sigma(\lfloor \frac{j+1}{2} \rfloor)+\alpha\right)}  \}  \subseteq  V^{l_0}_{k_1} \cup V^{l_0}_{k_2} \subseteq V^{l_0+1}_{\max\{k_1,  k_2\}}.
\end{equation*}
This completes the proof of our assertion.

Set the reasoning start as  $x_{2s+1} = x_j$,   
we then assert that 
\begin{equation}\label{assertion2}
  \bigcup\limits_{0\leqslant \alpha \leqslant 2^{l-1}-2 }\{ x_{2 \sigma^{-1}\left( \sigma(\lfloor \frac{j+1}{2} \rfloor)  + \alpha\right) -1} ,     x_{2 \sigma^{-1}\left(\sigma(\lfloor \frac{j+1}{2}\rfloor ) + \alpha \right) } \} \subseteq V^l_{2s+1}.
\end{equation}

When $l=2$ the assertion \eqref{assertion2} can be easily verified.  We assume \eqref{assertion2} holds for $l= l_0$ and consider the case $l=l_0+1$.
By assertion \eqref{finite_inductive_hypotheses},   we know there exists $N^{l_0}_{k}$ such that 
\begin{equation} 
\bigcup\limits_{ 0 \leqslant \alpha \leqslant 2^{l_0-1}-1} \{ x_{2\sigma^{-1}\left(\sigma(\lfloor \frac{j+1}{2} \rfloor)+2^{l_0-1}-1+\alpha\right)-1},  x_{2\sigma^{-1}\left(\sigma(\lfloor \frac{j+1}{2} \rfloor)+2^{l_0-1}-1+\alpha\right)}  \} \subseteq V^{l_0}_k.
\end{equation}

Again,   by definition of reasoning sequence we have 
\begin{equation*} 
  x_{2\sigma^{-1}\left(\sigma(\lfloor \frac{j+1}{2} \rfloor) + 2^{l_0-1} -1\right)-1} = x_{2\sigma^{-1}\left(\sigma(\lfloor \frac{j+1}{2} +\rfloor) + 2^{l_0-1}-2 \right)}.
\end{equation*}
And by Rule 3 we know that 
\begin{equation*}
  \bigcup\limits_{ 0 \leqslant \alpha \leqslant  2^{l_0}-2} \{ x_{2\sigma^{-1}\left(\sigma(\lfloor \frac{j+1}{2} \rfloor)+\alpha  \right)-1},  x_{2\sigma^{-1}\left(\sigma(\lfloor \frac{j+1}{2} \rfloor)+\alpha\right)}  \}  \subseteq V^{l_0}_k \cup V^{l_0}_{2s+1} \subseteq V^{l_0+1}_{2s+1}.
\end{equation*}
This completes the proof of assertion \eqref{assertion2},   which implies that for $l \geqslant 2$, $C^{l}_{2s+1} \geqslant 2^{l-1} $. 

For $l=1$, it is clear that $C^2_{2s+1} = 1$ since only residual connection happens.  And we complete the proof for the lower bound.

The upper bound is proved similarly as in the proof of Theorem \ref{Information_quantity_theorem}, except we consider mainly the reasoning start position, which results in one fewer element.

\section{Examples where the theoretical bounds are achieved}\label{Examples}
In this section we give some examples related to Theorem \ref{finite_sequence_theorem}. In fact, both the lower bound and upper bound can be attained.
\subsubsection*{Lower bound}
We construct a reasoning sequence $(1,  2,  2,  3,  3,  4,  4,  5\cdots,  2s-1,  2s,  1)$. Assume that $1\leqslant l \leqslant 1+ \log_2 s$,   it is clear that $C^{l}_{2s+1} \geqslant 2^{l-1}$

\begin{figure}[tbp] 
    \centering 
    \includegraphics[width=0.8\linewidth]{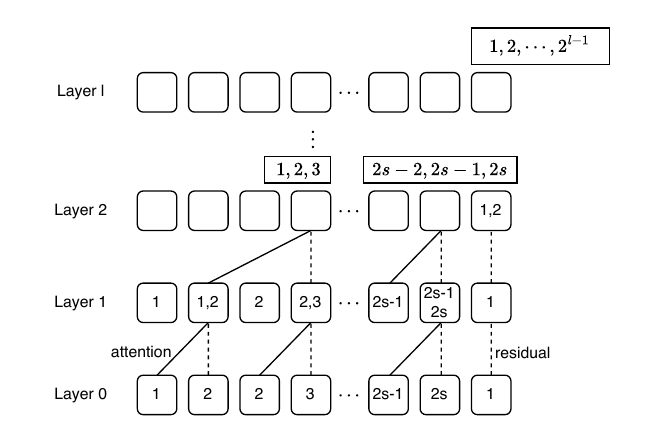} 
    \caption{Lower bound example} 
    \label{fig:Lower_bound_example} 
\end{figure}

\subsubsection*{Upper bound}
In this section, we shall use the concept of truncation of a reasoning sequence. This truncated reasoning sequence contains a finite step reasoning chain and some irrelevant reasoning pairs which serves as some redundant information as in practice. In fact, we truncate a reasoning sequence $(x_i)$ as follows:
\begin{enumerate}
  \item Firstly,   we truncate the constructing reasoning sequence $(\boldsymbol{a}_m)$ to be an $s$ step reasoning chain $ (\boldsymbol{\tilde{a}}_k) =  (\boldsymbol{a}_i,  \boldsymbol{a}_{i+1},  \cdots,  \boldsymbol{a}_{i+s-1})$.

  \item Secondly,   we use the relation \eqref{ax_relation} to find those elements constructed from this $s$ step reasoning chain. That is  
  \begin{equation}
    E = \{ x_{2 (\sigma^{-1}(i)-1)},    x_{2 (\sigma^{-1}(i)-1)+1},  \cdots,   x_{2 (\sigma^{-1}(i+s-1)-1)},    x_{2 (\sigma^{-1}(i+s-1)-1)+1} \}. \label{truncation_chain_index}
  \end{equation}

  \item Thirdly,   we set 
  \begin{equation}
  \begin{aligned}
      \mathcal{I}_E = \{ &{2 (\sigma^{-1}(i)-1)},  {2 (\sigma^{-1}(i)-1)+1},   \cdots, \\  &{2 (\sigma^{-1}(i+s-1)-1)},    {2 (\sigma^{-1}(i+s-1)-1)+1} \}\label{truncation_index},
  \end{aligned}
  \end{equation}
  which is the set of all subscripts of elements in $E$. Moreover, we take the subsequence $(\tilde{x}_i) = (x_{\min\{ \mathcal{I}_E \}+i-1})$ as the truncated reasoning sequence,   
  where  ${1\leqslant i \leqslant \max\{ \mathcal{I}_E \} - \min\{ \mathcal{I}_E \} +1 }$. 

\end{enumerate}

We  call this sequence $(\tilde{x}_i)$ the truncated reasoning sequence containing $(\boldsymbol{\tilde{a}}_k)$.
  
With this definition of truncation, we provide a method to construct  finite step reasoning sequences which allows the upper bound in Theorem \ref{finite_sequence_theorem} to be attained.
This is achieved in a recursive way. 

Consider a reasoning sequence $(\boldsymbol{a}_i)_{i \in \mathbb{Z}} $. For a given integer $i$
we define  
\begin{equation*}
  \vs_{1}(i) =(\boldsymbol{a}_i), \ \  \vs_{2}(i)=(\boldsymbol{a}_i,  \boldsymbol{a}_{i+2},  \boldsymbol{a}_{i+1}),  
\end{equation*}
and 
\begin{equation*}
  \vs_{3}(i)  = (\boldsymbol{a}_i,  \boldsymbol{a}_{i+2},  \boldsymbol{a}_{i+1},   \boldsymbol{a}_{i+6},  \boldsymbol{a}_{i+8},  \boldsymbol{a}_{i+7},   \boldsymbol{a}_{i+3},  \boldsymbol{a}_{i+5},  \boldsymbol{a}_{i+4} ).
\end{equation*}
For simplicity,   we use the notation 
\begin{align*}
 \vs_{3}(i)&=(\vs_{2}(i),  \vs_{2}(i+6),  \vs_{2}(i+3) )\\
 &:=(\boldsymbol{a}_i,  \boldsymbol{a}_{i+2},  \boldsymbol{a}_{i+1},   \boldsymbol{a}_{i+6},  \boldsymbol{a}_{i+8},  \boldsymbol{a}_{i+7},   \boldsymbol{a}_{i+3},  \boldsymbol{a}_{i+5},  \boldsymbol{a}_{i+4} ),  
\end{align*}
and with this notation $\vs_{3}(i)$ is still a sequence of reasoning pairs.

For $k \geqslant 3$ we define
\begin{equation*}
  \vs_{k}(i)=(\vs_{k-1}(i),    \vs_{k-1}(i+2 \times 3^{k-2}),  \vs_{k-1}(i+3^{k-2})).
\end{equation*}
Same notation as in $\vs_{3}(i)$ and all $\vs_{k}(i)$ are sequences of reasoning pairs.

We now construct our reasoning sequence.

Set $r_1 = (\vs_{l-1}(i-\frac{3^{l-1}-1}{2}),\cdots,\vs_3(i-13),\vs_2(i-4), \vs_1(i-1),\vs_1(i), \vs_2(i+1),  \vs_3(i+4),   \cdots,   \vs_{l-1}(i+  \frac{3^{l-2}-1}{2}) )$,   
it is clear that there exist $\sigma \in \text{Sym}( \mathbb{Z})$ such that $(\boldsymbol{a}_{\sigma(1-\frac{3^{l-1}}{2})},\cdots,\boldsymbol{a}_{\sigma(1)},   \boldsymbol{a}_{\sigma(2)},  \cdots,  \boldsymbol{a}_{\sigma(\frac{3^{l-1}-1}{2})}) = r_1$.
By the construction of $r_1$ we know all the reasoning pairs in $r_1$ forms a finite step reasoning chain 
$(\boldsymbol{a}_{i-\frac{3^{l-1}-1}{2}},\cdots\boldsymbol{a}_{i},  \boldsymbol{a}_{i+1},  \cdots,  \boldsymbol{a}_{i+\frac{3^{l-1}-3}{2}})$,   we extend this finite step  reasoning chain to an infinite reasoning chain $\tilde{a}$ by adding reasoning pairs to both sides.
And we take a permutation $\tau  \in \text{Sym}( \mathbb{Z})$ satisfying $\boldsymbol{a}_{\tau(i_m)}  = \boldsymbol{a}_{\sigma(m)} $ where $i_m \in I$ with $I = \{i_{1-\frac{3^{l-1}}{2}},\cdots ,i_1,  i_2,  \cdots,  i_{\frac{3^{l-1}-3}{2}}\}$  and $i_{1-\frac{3^{l-1}}{2}}<\cdots<i_{1}<i_{2}<\cdots<i_{\frac{3^{l-1}-3}{2}}$. 

We truncate the reasoning sequence $((x_j),   \tilde{\va},   \tau)$ to contain the finite reasoning chain 
$(\boldsymbol{a}_{i-\frac{3^{l-1}-1}{2}},  \cdots,  \boldsymbol{a}_{i+\frac{3^{l-1}-3}{2}})$
to get a sequence $(\tilde{x}_j) = (x_{\min\{ \mathcal{I}_E \}+j-1})_{1\leqslant j \leqslant \max\{ \mathcal{I}_E \} - \min\{ \mathcal{I}_E \} +1 }$. Here $E$ and $\mathcal{I}_{E}$ are defined as in \eqref{truncation_chain_index} and \eqref{truncation_index}.  
This $(\tilde{x}_j)$ with reasoning start $x_{\max{\mathcal{I}_E}+1}= \boldsymbol{a}^1_i$ is the sequence we want.

\begin{remark}
  Under the rules of information propagation,   the elements in $(\tilde{x}_j)$ contributing to the node $N^l_{\max{\mathcal{I}_E}+1}$ are those come from the reasoning chain 
  $(\boldsymbol{a}_{i-\frac{3^{l-1}-1}{2}},  \cdots,  \boldsymbol{a}_{i+\frac{3^{l-1}-3}{2}})$.
  Hence,   the information in $N^l_{\max{\mathcal{I}_E}+1}$ will be the same if we consider the reasoning sequence constructed from 
  $(\boldsymbol{a}_{i-\frac{3^{l-1}-1}{2}},  \cdots,  \boldsymbol{a}_{i+\frac{3^{l-1}-3}{2}})$
  and $\sigma$. 
  However,    the above complicated  way we construct truncated sequence is still necessary which shows that there is a large class of reasoning sequence allows the upper bound to be attained.
\end{remark}

\begin{proposition}\label{upper_bound_example_prop}
  We consider the reasoning sequence $(x_i)$ with constructing reasoning chain $(\boldsymbol{a}_m)=((m,  m+1))_{1-\frac{3^{\tilde{l}-1}-1}{2}\leqslant m \leqslant \frac{3^{\tilde{l}-1}-1}{2}}$ and constructing permutation $\sigma$ satisfying that 
  \begin{align}
      &(\boldsymbol{a}_{\sigma(1)},  \boldsymbol{a}_{\sigma(2)},  \cdots,  \boldsymbol{a}_{\sigma(\frac{3^{\tilde{l}-1}-1}{2})}) = (\vs_1(1),  \vs_2(2),     \vs_3(5),  \cdots,  \vs_{l}(\frac{3^{l-1}+1}{2}),\cdots \vs_{\tilde{l}-1}(\frac{3^{\tilde{l}-2}+1}{2}) ),\\
      &
      \begin{aligned}
          (\boldsymbol{a}_{\sigma(1-\frac{3^{\tilde{l}-1}}{2})},  \boldsymbol{a}_{\sigma(2-\frac{3^{\tilde{l}-1}}{2})},  \cdots,  &\boldsymbol{a}_{\sigma(0)})=\\ &(\vs_{\tilde{l}-1}(\frac{3-3^{\tilde{l}-1}}{2}),\cdots,\vs_l(\frac{3-3^l}{2}),\cdots\vs_3(-12),\vs_2(-3),\vs_1(0),).
      \end{aligned}
  \end{align}
  If the reasoning start is set be to $x_{2\sigma(\frac{3^{\tilde{l}-1}-1}{2})+1} = \boldsymbol{a}^1_1$,   then for $1\leqslant l \leqslant \tilde{l}$,   we have $C^l_{2\sigma(\frac{3^{\tilde{l}-1}-1}{2})+1} = 3^{l-1}$.
\end{proposition}
To prove this proposition we need the following two lemmas.
\begin{lemma}\label{upper_bound_lemma}
  For $m\geqslant 1$,  $ \forall k \in \{ \frac{3^{j-2}+1}{2} \}_{2\leqslant j \leqslant \tilde{l}}$,  there exists $i_{(m,k)}\in [1,{3^{\tilde{l}-1}-1}] \cap \mathbb{Z}$ depending on $m$ and $k$  such that the node $N^m_{i_{(m,k)}}$ have value set $V^m_{i_{(m,k)}}$ satisfying the following property.
  \begin{equation}
    \{k,k+1,\cdots,k+3^{m-1} \} \subseteq V^m_{i_{(m,k)}}.
  \end{equation}
\end{lemma}

\begin{lemma}\label{left_upper_bound_lemma}
  For $m\geqslant 1$,  $ \forall k \in \{ \frac{3-3^{j-1}}{2} \}_{2\leqslant j \leqslant \tilde{l}}$,  there exist $i_{(m,k)}\in [1-{3^{\tilde{l}-1}},0] \cap \mathbb{Z}$ such that the node $N^m_{i_{(m,k)}}$ have value set $V^m_{i_{(m,k)}}$ satisfying the following property.
  \begin{equation}
    \{k,k+1,\cdots,k+3^{m-1} \} \subseteq V^m_{i_{(m,k)}}.
  \end{equation}
\end{lemma}

\begin{proof}[\textbf{Proof of Proposition \ref{upper_bound_example_prop} }]
  The case $l=1$ is trivial and omitted. We use mathematical induction to prove this proposition. 
  We assert that for $l \geqslant 2$  the reasoning start node $N^l_{2\sigma(\frac{3^{\tilde{l}-1}-1}{2})+1}$ have value set $V^l_{2\sigma(\frac{3^{\tilde{l}-1}-1}{2})+1}$ satisfying
  \begin{equation}
    \{\frac{3-3^{l-1}}{2},\frac{3-3^{l-1}}{2}+1,\cdots,0, 1,2,\cdots, \frac{3^{l-1}+1}{2} \} \subseteq V^l_{2\sigma(\frac{3^{\tilde{l}-1}-1}{2})+1}.
  \end{equation}
  The case $l=2$ can be verified easily. 
  We assume that  for $ l =l_0 \in [2,\tilde{l}-1]\cap \mathbb{Z}$  the node $N^{l_0}_{2\sigma(\frac{3^{\tilde{l}-1}-1}{2})+1}$ have value set $V^{l_0}_{2\sigma(\frac{3^{\tilde{l}-1}-1}{2})+1}$ satisfying
 \begin{equation}
    \{\frac{3-3^{l_0-1}}{2},\frac{3-3^{l_0-1}}{2}+1,\cdots,0, 1,2,\cdots, \frac{3^{l_0-1}+1}{2} \} \subseteq V^l_{2\sigma(\frac{3^{\tilde{l}-1}-1}{2})+1}.
  \end{equation}
  Now for $l=l_0 +1 $, by Lemma \ref{upper_bound_lemma}, we know there exists a node $N^{l_0}_{i(l_0, \frac{3^{l_0-1}+1}{2})} $ such that 
  \begin{equation*}
    \{\frac{3^{l_0-1}+1}{2},\cdots, \frac{3^{l_0-1}+1}{2} + 3^{l_0-1}  \}  \subseteq  V^{l_0}_{i(l_0, \frac{3^{l_0-1}+1}{2})}.
  \end{equation*}

   And by Lemma \ref{left_upper_bound_lemma} there exist   a node $N^{l_0}_{i(l_0, \frac{3-3^{l_0}}{2})} $ such that 
  \begin{equation*}
    \{\frac{3-3^{l_0}}{2},\cdots, \frac{3-3^{l_0}}{2} + 3^{l_0-1}   \}  \subseteq  V^{l_0}_{i(l_0, \frac{3-3^{l_0}}{2})}.
  \end{equation*}

  We then have 
  \begin{equation}
    \begin{aligned}
      V^{l_0}_{i(l_0, \frac{3^{l_0-1}+1}{2})} \cap  V^{l_0}_{2\sigma(\frac{3^{\tilde{l}-1}-1}{2})+1} = \{\frac{3^{l_0-1}+1}{2}\},\\
      V^{l_0}_{i(l_0, \frac{3-3^{l_0}}{2})}\cap V^{l_0}_{2\sigma(\frac{3^{\tilde{l}-1}-1}{2})+1}  = \{\frac{3-3^{l_0-1}}{2}\},
    \end{aligned}
  \end{equation}
and by Rule 3, we know that the node $N^{l_0+1}_{2\sigma(\frac{3^{\tilde{l}-1}-1}{2})+1} $ have value set $V^{l_0+1}_{2\sigma(\frac{3^{\tilde{l}-1}-1}{2})+1} $ satisfying
   \begin{equation}\label{temp_r1}
   \begin{aligned}
        \{\frac{3-3^{l_0}}{2} \cdots, 0,1,2,\cdots, \frac{3^{l_0}+1}{2}\} &= V^{l_0}_{i(l_0, \frac{3^{l_0-2}+1}{2})}  \cup V^{l_0}_{2\sigma(\frac{3^{\tilde{l}-1}-1}{2})+1}  \cup V^{l_0}_{i(l_0, \frac{3-3^{l_0-1}}{2})} \\ & \subseteq V^{l_0+1}_{2\sigma(\frac{3^{\tilde{l}-1}-1}{2})+1}.
   \end{aligned}
   \end{equation}

   This completes the proof of our assertion. From above assertion we know that $C^l_{2\sigma(\frac{3^{\tilde{l}-1}-1}{2})+1} \geqslant 3^{l-1}$. Combining the proof of Theorem \ref{finite_sequence_theorem}, we know that $C^l_{2\sigma(\frac{3^{\tilde{l}-1}-1}{2})+1} \leqslant 3^{l-1}$. And therefore 
  $C^l_{2\sigma(\frac{3^{\tilde{l}-1}-1}{2})+1}=  3^{l-1}$.
\end{proof}

\begin{proof}[\textbf{Proof of Lemma \ref{upper_bound_lemma} }]
  We use mathematical induction to prove this lemma. The case $m=2$ can be verified easily through Rule 2. We assume that for $m=m_0$,  $\forall k \in \{  \frac{3^{l-2}+1}{2} \}_{2\leqslant l \leqslant \tilde{l}}$,
   \begin{equation}
    \{k,k+1,\cdots,k+3^{m_0-1} \} \subseteq V^{m_0}_{i_{({m_0},k)}}.
  \end{equation}
  For $m=m_0+1$, by assumption, there exists $i(m_0,k)$, $i(m_0,k+3^{m_0-1})$,   $i(m_0,k+2\times3^{m_0-1})$ such that 
  \begin{equation}\label{lemma_temp}
    \begin{aligned}
      &\{ k,\cdots, k+3^{m_0-1}\} \subseteq V^{m_0}_{i(m_0, k)},\\
      &\{ k+3^{m_0-1},\cdots, k+2\times 3^{m_0-1}\} \subseteq V^{m_0}_{i(m_0, k+3^{m_0-1} )},\\
      &\{ k+2\times 3^{m_0-1},\cdots, k+3^{m_0}\}\subseteq V^{m_0}_{i(m_0, k+2\times3^{m_0-1} )}.
    \end{aligned}
  \end{equation}
  Set $i{(m_0 +1 , k)} = \max \{{i(m_0, k)},i(m_0, k+3^{m_0-1} ), i(m_0, k+2\times3^{m_0-1} )\}$, by Rule 3 and \eqref{lemma_temp} we know that 
  \begin{equation}
     \{k,\cdots, k+3^{m_0} \} \subseteq V^{m_0}_{i(m_0, k)}\cup V^{m_0}_{i(m_0, k+3^{m_0-1} )  }\cup V^{m_0}_{i(m_0, k+2\times3^{m_0-1} )} \subseteq V^{m_0 +1 }_{i(m_0 +1 , k) },
  \end{equation}

which completes the proof.
\end{proof}
The proof of Lemma \ref{left_upper_bound_lemma} is analogous to that of Lemma \ref{upper_bound_lemma} and is omitted.

\section{Construction of transformer}\label{Construction_of_Transformer}
\subsection{Embedding} 
We assume that $d_{m}$ is large enough. 
We choose a suitable embedding
(which can be done by choosing a suitable basis of $\mathbb{R}^{d_m}$) 
such that the non-zero elements of all value vector $\mX^{pos}$ are located in the first $n$ coordinates and all the elements in $\mX^{tgt}$ located at the first $n$ coordinates are zero.  

We denote $\mX^{tgt}= (\mX^{tgt,\mathsf{T}}_1,\mX^{tgt,\mathsf{T}}_2,\cdots,\mX^{tgt,\mathsf{T}}_n)^{\mathsf{T}}$ with each $\mX^{tgt}_i= (\mX^{tgt}_{i,1},\mX^{tgt}_{i,2},\cdots,\mX^{tgt}_{i,d_m})\in \mathbb{R}^{d_m}$. Similarly, we denote $\mX^{pos}= (\mX^{pos,\mathsf{T}}_1,\mX^{pos,\mathsf{T}}_2,\cdots,\mX^{pos,\mathsf{T}}_n)^{\mathsf{T}}$ and  $\mX^{(l)} = (\mX^{(l),\mathsf{T}}_1,\mX^{(l),\mathsf{T}}_2,\cdots,\mX^{(l),\mathsf{T}}_{n})^{\mathsf{T}}$.

We correspond each element in the set $\tilde{E}$ to an index i.e., the set $\tilde{E}$ is of the form $\tilde{E}=\{\vv_1,\vv_2,\cdots,\vv_d\}$. Moreover, each $\vv_i$ is of the form $(\underbrace{0, 0, \cdots, 0}_{n\  \text{zeros}},0,\cdots,1,\cdots,0)$, and we denote $  k_i $ the position of $1$. 
Let $n\leqslant k_{1}< k_{2}<\cdots <k_{d}$, we require the embedding is chosen such that ${k_{i}}$ satisfies that 

\begin{equation}\label{embedding_requirement}
  \left\{ \begin{aligned}
    &k_{1}-n \geqslant 2(n+1)(3^{L}+1),\\
    &k_{i}-k_{i-1}  \geqslant 2(n+1)(3^{L}+1), \ \  \text{ for } 2 \leqslant i\leqslant d,\\
    &d_{m} -k_d\geqslant 2(n+1)(3^{L}+1).
  \end{aligned}\right.
\end{equation}

\subsection{Construction of parameters}\label{Construction_of_parameters}
We choose the positional encodings to be the first $n$ standard-basis vectors, $\vp_i=e_i$. The spacing condition above applies to the entire dictionary, so position blocks remain disjoint for every token choice; this will also allow a shared FNN in Proposition~\ref{prop:shared_position_ffn}.
We set the  weight matrices as follows:

\begin{align}
  &\mW^{q(l)}=\mI(0 \leqslant l \leqslant L),\label{ql}\\
  &\mW^{vo(l)}=\mI(1\leqslant l\leqslant L), \  \mW^{vo(0)}=\mR,\label{vol}\\
  &\mW^{k(l),\mathsf{T}}=\sum_{i=-(n+1)3^{L}}^{-1} \mR^{i}(l\geqslant 1), \ \ \mW^{k(0)}=\sum_{i=1}^{\frac{n-1}{2}}\vp_{2i-1}\vp_{2i} ^{T}.\label{kl}
\end{align}

Here $\vp_{j}$ is  the positional encoding and  arbitrary two positional encodings are orthogonal to each other. Moreover, $\mR \in \mathbb{R}^{d_m\times d_m}$ is defined as 
 $\mR=[\vR_{ij}]$, with $ \vR_{i+1,  i}=1=\vR_{1d_{m}} \text{ for } 1\leqslant i \leqslant d_{m}-1  $, and all the other elements of $\mR$ are set to be 0. In fact,

$$
\mR = \begin{bmatrix}
0 & 0 & 0 & \cdots & 1 \\
1 & 0 & 0 & \cdots & 0 \\
0 & 1 & 0 & \cdots & 0 \\
\vdots & \vdots & \ddots & \ddots & \vdots \\
0& 0 & \cdots & 1 & 0
\end{bmatrix}.
$$

This matrix $\mR$ is called the left shift matrix.  When we  apply $\mR$  to a vector $\vv = (\vv_1,\vv_2,\cdots,\vv_{d_m}) $ then  result is $\vv\mR =(\vv_2,\cdots,\vv_{d_m-1},\vv_{d_m},\vv_1) $. 
The notation $\mR^i$ means that the matrix $\mR$  act $i$ times, i.e., $\vv \mR^i = \vv \mR \mR \cdots \mR $. Also, the inverse of $mR$ is called the right shift matrix, and $\mR^{-1} = \mR^{\mathsf{T}}$.

For any given matrix $\mM  = [m_{ij}]$, we define the mask operation as 
\begin{equation}
  \text{mask}(\mM) = \begin{bmatrix}
    m^{(l)}_{1,  1} & -\infty & -\infty & \cdots & -\infty \\
m^{(l)}_{2,  1} & m^{(l)}_{2,  2} & -\infty & \cdots & -\infty \\
m^{(l)}_{3,  1} & m^{(l)}_{3,  2} & m^{(l)}_{3,  3} & \cdots & -\infty \\
\vdots & \vdots & \vdots & \ddots & \vdots \\
m^{(l)}_{n,  1}& m^{(l)}_{n,  2} & \cdots & m^{(l)}_{n,  n-1} & m^{(l)}_{n,  n}
  \end{bmatrix}.
\end{equation}

Furthermore, we denote
\begin{equation}\label{Al}
  \mA^{(l)}= \mathrm{mask}(\mX^{(l)}\mW^{q(l)}\mW^{k(l),  \mathsf{T}}\mX^{(l),  \mathsf{T}})=
  \begin{bmatrix}
\mA^{(l)}_{1,  1} & -\infty & -\infty & \cdots & -\infty \\
\mA^{(l)}_{2,  1} & \mA^{(l)}_{2,  2} & -\infty & \cdots & -\infty \\
\mA^{(l)}_{3,  1} & \mA^{(l)}_{3,  2} & \mA^{(l)}_{3,  3} & \cdots & -\infty \\
\vdots & \vdots & \vdots & \ddots & \vdots \\
\mA^{(l)}_{n,  1}& \mA^{(l)}_{n,  2} & \cdots & \mA^{(l)}_{n,  n-1} & \mA^{(l)}_{n,  n}
\end{bmatrix}.
\end{equation}

We shall still use the notation $N^{l}_i$ to denote the $i$th node in $l$th layer with slightly modification: the value set $\mV^{l}_i$ is now a set $\mV^{l}_i$ which contains the embedded value as element.

Given any reasoning chain $(\va_u)_{1\leqslant i \leqslant n}$, after embedding it is of the form

$$(\tilde{\va}_u) = ( (\tilde{\va}^1_u,\tilde{\va}^2_u)  )_{1\leqslant i \leqslant n} = ( ({\va}^1_u\mW^{\operatorname{emb}},{\va}^2_u\mW^{\operatorname{emb}})  )_{1\leqslant i \leqslant n}. $$

We define a sequence $(\tilde{\vb}_v)_{1\leqslant v \leqslant n+1} $ as follows:
\begin{equation}
  \begin{aligned}
    &\tilde{\vb}_v = \tilde{\va}^1_v, \ \ \text{for } 1 \leqslant v \leqslant n,\\
    &\tilde{\vb}_{v} = \tilde{\va}^2_{n}\ \ \text{for } v= n+1.
  \end{aligned}
\end{equation} 

Assume the node $N^l_i$ contains the information from $(\tilde{a}_i)$, then its value set $\mV^l_i$ must contain the elements from a subsequence of $\tilde{\vb}_{v \in {[v_1,v_2] \cap \mathbb{Z}}}$ where $v_1,v_2$ depends on $i$ and $l$. 
For simplicity, we set  $\vb^l_{i,v} = \tilde{\vb}_{v_1 + v -1 } $  for $v \in [1,v_2-v_1+1]\cap \mathbb{Z}$. In fact, $v_2 - v_1 +1  = C^l_i$. 

With these notations, we define the FNN as follows 

\begin{equation}\label{output}
  \left\{\begin{aligned}
    R^{l}_{N}(f^{l}_{i}(\mX^{ao(l)}_{i})+(\mX^{ao(l)}_{i}))& = \sum_{1\leqslant k\leqslant C^{l}_i} \vb_{i;k}\mR^{i3^{L}-j+k} \text{ for } i \geqslant 2, \\
   R^{l}_{N}(f^{l}_{i}(\mX^{ao(l)}_{i})+(\mX^{ao(l)}_{i}))&=\mX^{tgt}_{1}\mR^{(n+1)3^{L}}\ \text{for }\ i=1.
  \end{aligned}
  \right.
\end{equation}
Here $R_N$ denotes RMSNorm with a positive stability constant $\epsilon$ and zero bias; we choose a scalar gain $\alpha>1$. Its definition and injectivity are given in Lemma~\ref{rmsnorm_injective}. The index $j\in [1,n]\cap\mathbb{Z}$ satisfies  $\vb^l_{i,j}= X^{tgt}_i$.
The subscript $i$ on the local FNN rules specifies their input domains. Proposition~\ref{prop:shared_position_ffn} shows that these rules can be realized, to arbitrary accuracy, by one FNN shared across all positions in each layer.

\begin{remark}\label{impossible}

  Note that this is not an accurate expression, since we can not simply define the output of a FNN. However, we can show that there exist $f^{(l)}_i$ for each layer such that the output of  $R_{N}(f^{l}_{i}(\mX^{ao(l)}_{i})+(\mX^{ao(l)}_{i}))$ differs very little from the right-hand side of \eqref{output}. 
And also the information propagation won't be affected by this error. This will be explained in detail in later sections.
\end{remark}
By this construction the reasoning information contained in the node $N^l_i$ is encoded by $R_{N}(f^{l}_i(\cdot)+(\cdot))$.

In order to extract the result after m-step reasoning, we set 

\begin{equation}
  \mW^{p}=\mR^{-n3^{L}-m+1}\mQ \mW^{emb,\mathsf{T}},
\end{equation}
where $\mQ$ satisfies $\mQ \vv_{1}=\vv_{1}$, ..., $\mQ \vv_{d}=\vv_{d}$ and maps other basis to 0.
Note that last token on the $s$-th layer is $\mX_{n}^{(L)}=\vb_{n;1}\mR^{n3^{s}-j+1}+......+\vb_{n;t}\mR^{n3^{s}+t-j}$,  and we have 

\begin{equation}
  \mX_{n}^{(L)}\mW^{p}=\vb_{n;1}\mR^{-j-m+2}\mQ \mW^{emb,\mathsf{T}}+......+\vb_{n;t}\mR^{t-j-m+1}\mQ \mW^{emb,\mathsf{T}},
\end{equation}
and the output 
\begin{equation}
  \mY=argmax(\tilde{\sigma}(\mX_{n}^{(L)}\mW^{p})).
\end{equation}

When $m+j-1\leqslant t$, $\mX_{n}^{(L)}\mW^{p}=\vb_{n;m+j-1}\mW^{emb,\mathsf{T}}$ and the above setting yields the desired reasoning result.
However, when $m+j-1>t$, $\mX_{n}^{(L)}\mW^{p}=0$ and therefore we can not get a right reasoning result. 

In fact, the sufficiency of transformer depth $s$ relative to the required reasoning steps $m$ is a key factor in ensuring accurate reasoning result.

We can roughly categorize this relationship into three cases 
\begin{itemize}
  \item \textbf{Case 1} $m\leqslant 2^{L-1}-1$;
  \item \textbf{Case 2} $2^{s-1}\leqslant m\leqslant \frac{3^{L-1}+1}{2}$;
  \item \textbf{Case 3} $m>\frac{3^{L-1}-1}{2}$.
\end{itemize}
For the first case, note that $t-j+1\ge 2^{L-1}-1$ and therefore $m+j-1\le t$ which means we can get the result.
For the third case, since $t-j+1\le\frac{3^{L-1}+1}{2}$  and consequently $m+j-1>t$, the model can not derive the  result. 
The second case is more complicated, since we can not derive the relationship of $t-j+1$ and $t$ from the relationship of $m$ and $s$. Whether the model can derive the reasoning result now depends on $t-j+1$ and $t$.

\subsection{Explanation}

First, we explain how the same token matching rule works in this construction. More specifically, the attention matrices defined above satisfy the following property.
\begin{lemma}
For $l \geqslant 1$, we have 
\begin{equation}
  \begin{aligned}
    &\mA^{(l)}_{i,  j}=0, \text{ if }\ \ j=1,\\
    &\mA^{(l)}_{i,  j}=0, \text{ if }\ \  i\geqslant j\geqslant2, \ \mV^{l}_{i} \cap \mV^{l}_{j} =\emptyset,\\
    &\mA^{(l)}_{i,  j}\geqslant1, \text{ if }\ \ i\geqslant j\geqslant2,\ \mV^{l}_{i} \cap \mV^{l}_{j} \neq \emptyset.
  \end{aligned}
\end{equation}
\end{lemma}

\begin{proof}
Start with given any $\mX^{(0)}= (\mX^{(0)}_1,\mX^{(0)}_2,\cdots,\mX^{(0)}_n)^{\mathsf{T}}$, we denote $\mX^{(l)}=(\mX^{(l)}_1,\mX^{(l)}_2,\cdots,\mX^{(l)}_n) $ with $\mX^{(l)}_{i}=\sum_{u=1}^{C^{l}_{i}} \vb^{l}_{i;u}\mR^{i3^{L}+u-d^{(l)}_{i}}$, 
where $C^{l}_i = |\mV^{l}_i|$ and  $\vb^{l}_{i;u} \in \tilde{E}$. 

\begin{align}
\mA^{(l-1)}_{i,  j}&=\mX^{(l-1)}_{i}\mR^{q(l-1)}\mR^{k(l-1),  \mathsf{T}} \mX_{j}^{(l-1),  \mathsf{T}}\\
     &=(\sum_{u=1}^{C^{l-1}_{i}} \vb_{i;u}^{l-1}\mR^{i3^{L}+u-d^{(l-1)}_{i}})(\sum_{m=-(n+1)3^{L}}^{-1} \mR^{m})(\sum_{v=1}^{C^{l-1}_{j}} \vb_{j;v}^{l-1}\mR^{j3^{L}+v-d^{(l-1)}_{j}})^{\mathsf{T}}\\
     &=\sum_{u=1}^{C^{l-1}_{i}}\sum_{m=-(n+1)3^{L}}^{-1}\sum_{v=1}^{C^{l-1}_{j}} \vb_{i;u}^{l-1}\mR^{i3^{L}+u-d^{(l-1)}_{i}-j3^{L}-v+d^{(l-1)}_{j}+m}\vb^{l-1,  \mathsf{T}}_{j;v}
\end{align}
For any $i\geqslant j \geqslant 2$ and   $\mV^{l-1}_{i} \cap \mV^{l-1}_{j} =\emptyset$. 

Since $-(n+1)3^{L}\leqslant i3^{L}+u-d^{(l-1)}_{i}-j3^{L}-v+d^{(l-1)}_{j}+m\leqslant(n+1)3^{L}$ and $\vb_{i;u}^{l}\ne \vb^{l}_{j;v}$,  we have $\vb_{i;u}^{l-1}\mR^{i3^{L}+u-d^{(l-1)}_{i}-j3^{L}-v+d^{(l-1)}_{j}+m}\vb^{l-1,  \mathsf{T}}_{j;v}=0$.
Therefore,   $\mA^{(l-1)}_{i,  j}=0$.

When $i\geqslant j \geqslant 2$ and $\mV^{l-1}_{i} \cap \mV^{l-1}_{j} \ne \emptyset$,  

\begin{align}
\mA^{(l-1)}_{i,  j}&=\sum_{u=1}^{C^{l-1}_{i}}\sum_{m=-(n+1)3^{L}}^{-1}\sum_{v=1}^{C^{l-1}_{j}} \vb_{i;u}^{l-1}\mR^{i3^{L}+u-d^{(l-1)}_{i}-j3^{L}-v+d^{(l-1)}_{j}+m}\vb^{l-1,  \mathsf{T}}_{j;v}
\end{align}

Since $\mV^{l-1}_{i} \cap \mV^{l-1}_{j} \ne \emptyset$, there exists $j,v,i,v_0$ such that $\vb^{l-1}_{j;v}= \vb^{l-1}_{i;v_{0}}$. Moreover, since
$1\leqslant i3^{L}+u-d^{(l-1)}_{i}-j3^{L}-v+d^{(l-1)}_{j}+m\leqslant(n+1)3^{L}$ and there exists $j,v,i,v_0$ such that $\vb^{l-1}_{j;v}= \vb^{l-1}_{i;v_{0}}$,   
there exists $m_{0}$ s.t.  $\vb^{l-1}_{j;v}\mR^{i3^{L}+u_{0}-d^{(l-1)}_{i}-j3^{L}-v_{0}+d^{(l-1)}_{j}+m_{0}}\vb^{(l-1),  \mathsf{T}}_{j;v_{0}}=1$.
This indicates that   $\mA^{(l-1)}_{i,  j}\ge1$

When $j=1$,   

\begin{align}
\mA^{(l-1)}_{i,  1}= \sum_{u=1}^{C^{l-1}_{i}}\sum_{m=-(n+1)3^{L}}^{-1} \vb^{l-1}_{i;u}\mR^{i3^{L}+u-d^{(l-1)}_{i}-(n+1)3^{L}-1+m}\mX^{tgt,  \mathsf{T}}_{1}
\end{align}

Since $-2(n+1)3^{L}-1\leqslant i3^{L}+u-d^{(l-1)}_{i}-(n+1)3^{L}-1+m\leqslant -1$,  we know that $\mA^{(l-1)}_{i,  1}=0$.
\end{proof}
\subsection{Information propagation}

Next, we explain how the  above defined transformer extract reasoning result.

$\mX^{ao(0)}=\mX^{(0)}+\mX^{qkv(0)}$

When $l=0$, for $i\geqslant j$, using \eqref{ql}, \eqref{vol}, \eqref{kl} and \eqref{Al}  we have 
\begin{equation}
\begin{aligned}
 \mA_{i,j}^{(0)}&=\mX_{i}^{(0)}(\sum_{t=1}^{\frac{n-1}{2}}\vp_{2t}\vp_{2t-1}^{\mathsf{T}})\mX_{j}^{(0),\mathsf{T}}\\
    &=\sum^{\frac{n-1}{2}}_{t=1}\mX_{i}^{tgt}\vp_{2t}\vp_{2t-1}^{\mathsf{T}}\mX_{j}^{tgt,\mathsf{T}}+\sum^{\frac{n-1}{2}}_{t=1}\mX_{i}^{pos}\vp_{2t}\vp_{2t-1}^{\mathsf{T}}\mX_{j}^{tgt,\mathsf{T}}+\sum^{\frac{n-1}{2}}_{t=1}\mX_{i}^{pos}\vp_{2t}\vp_{2t-1}^{\mathsf{T}}\mX_{j}^{pos,\mathsf{T}}\\
    &+\sum^{\frac{n-1}{2}}_{t=1}\mX_{i}^{tgt}\vp_{2t}\vp_{2t-1}^{\mathsf{T}}\mX_{j}^{pos,\mathsf{T}}
\end{aligned}
\end{equation}
It is clear that 

\begin{equation}
  \begin{aligned}
    &\mA_{i,j}^{(0)}=1, \ \text{when $i=j+1$ and $i \bmod 2 =0$, }\\
    &\mA_{i,j}^{(0)}=0,\ \text{when $i\neq j+1$ or $i \bmod 2 \neq 0$. }
  \end{aligned}
\end{equation}

Then for $m\in [1,n]\cap \mathbb{Z} $, we get  $\mX_{m}^{ao(0)}=\sum_{i=1}^{m}\frac{\exp(\mA^{(0)}_{m,i})}{\sum_{j=1}^{m}\exp(\mA^{(0)}_{m,j})}\mX_{i}^{(0)}\mR+\mX_{m}^{(0)}$, when $m \bmod 2 =0$, we can get reasoning chain $(\vb_{m;1}^{(1)},\vb_{m;2}^{(1)})$. 

In fact, we  set all the coefficient $\frac{\exp(0)}{{\sum_{j=1}^{m}\exp(\mA^{(0)}_{m,j})}}$ to be $0$, then $\mX_{m}^{ao(1)}$ is of the form $\frac{\exp(1)}{\sum_{t=1}^{m}\exp(\mA^{(0)}_{m,t})}\vb_{m;1}^{(1)}\mR+\vb_{m;2}^{(1)}$.

By \eqref{output} we have  $R^{0}_{N}(f_{m}^{(0)}(\mX_{m}^{ao(0)})+\mX_{m}^{ao(0)})=b^{(1)}_{m;1}\mR^{m3^{L}-1}+b_{m;2}^{(1)}\mR^{m3^{L}}$

When $m \bmod 2 \neq 0$ and $  m>1$, we can recognize $b_{m;1}^{(0)}=\mX_{m}^{tgt}$, and $R^{0}_{N}(f_{m}^{(0)}(\mX_{m}^{ao(0)})+\mX_{m}^{ao(0)})=b_{m;1}^{(1)}\mR^{m3^{s}}$

When $m=1$, we can recognize $\vb_{1;1}^{(1)}=\mX_{1}^{tgt}$, and 
$R^{0}_{N}(f_{1}^{(0)}(\mX_{1}^{ao(0)}+\mX_{1}^{ao(0)}))=\vb_{1;1}^{(1)}\mR^{(n+1)3^{L}}$.

After the first layer decoder,  the even positions pass information to subsequent odd positions. And we can eliminate the position vectors on the second floor.

In fact, by our construction we have
\begin{align}
\mX^{ao(l-1)}_{i}&=\sum_{j=1}^{i}\frac{\exp(\mA^{(l-1)}_{i,  j})}{\sum_{t=1}^{i}\exp(\mA^{(l-1)}_{i,  t})} \mX_{j}^{(l-1)}+\mX^{(l-1)}_{i}\\
&=\sum_{j=1}^{i}\frac{\exp(\mA^{(l-1)}_{i,  j})}{\sum_{t=1}^{i}\exp(\mA^{(l-1)}_{i,  t})}\sum_{u=1}^{C^{l-1}_{j}}\vb^{l-1}_{j;u}\mR^{j3^{L}+u-d^{(l-1)}_{j}}+\sum_{u=1}^{C^{l-1}_{i}}\vb^{l-1}_{i;u}\mR^{i3^{L}+u-d^{(l-1)}_{i}}.
\end{align}

For any $i,\ j$  such that $i\ne j$, note that $|u-d^{(l-1)}_{j}|\leqslant \frac{3^{L}-1}{2}$ and $|v-d^{(l-1)}_{i}|\leqslant \frac{3^{L}-1}{2}$, by theorem \ref{finite_sequence_theorem},   we have 
\begin{align}
    |j3^{L}+u-d^{(l-1)}_{j}-i3^{L}-v+d^{(l-1)}_{i}|&=|(j-i)3^{L}+u-v-d^{(l-1)}_{j}+d^{(l-1)}_{i}|\\
    &\geqslant3^{L}-|u-v-d^{(l-1)}_{j}+d^{(l-1)}_{i}|\geqslant1,\\
    |i3^{L}+u-d^{(l-1)}_{i}|\leqslant(n+1)3^{L}.
\end{align}

And therefore, the nonzero element of $\vb^{l-1}_{j;u}\mR^{j3^{L}+u-d^{(l-1)}_{j}}$ and $\vb^{l-1}_{i;u}\mR^{i3^{L}+u-d^{(l-1)}_{i}}$ ($i\ne j$) won't locate at the same coordinate.

Denote $\mX^{ao(l-1)}_{i}=(\mX^{ao(l-1)}_{i,  1},  \cdots,    \mX^{ao(l-1)}_{i,  d_{m}})$. 
Firstly, set  the number on the axis which equals to $min_{1\leqslant j\leqslant i}\{\mX^{ao(l-1)}_{i,  j}>0\}$ to be $0$.

Suppose the nonzero component of   $X_{1}^{tgt}$ located at the $k_{p}$-th axis. 
Then, $X_{i,k_{p}-(n+1)3^{L}}^{ao(l-1)}=\frac{exp(0)}{\sum_{1\leqslant j\leqslant i}exp(A_{i,j}^{(l-1)})}$. 
Since $\min\limits_{1\leqslant j\leqslant i}\{\mX^{ao(l-1)}_{i,  j}>0\}=\frac{1}{\sum_{t=1}^{i}\exp(\mA^{(l-1)}_{i,  t})}$,   there remains the information propagated from $\mX^{(l-1)}_{j}$ s.t. $\mA^{(l-1)}_{i,  j} \geqslant 1$ which indicates that $\mV^{l-1}_{i} \cap \mV^{l-1}_{j} \neq \emptyset$.
We use the sequence $(\vb^{l-1}_{i;1},\cdots,\vb^{l-1}_{i;C^{l-1}_i})$ associated  to $N^{l}_i$ and the sequence $(\vb^{l-1}_{j;1},\cdots,\vb^{l-1}_{j;C^{l-1}_j})$ associated to $N^{l-1}_j$ to construct a new sequence $(\vb^{l}_{i})$
 associated to $N^{l}_i$ in the following two rules.
\begin{itemize}
  \item If $\mV_i \subseteq \mV_j $(resp. $\mV_j \subseteq \mV_i $), then we set $(\vb^{l}_i) = (\vb^{l-1}_i  )$(resp. $(\vb^{l}_i) = (\vb^{l-1}_j  )$). \label{chain_rule_1}
  \item If $\mV_i \not\subseteq \mV_j $ and $\mV_j \not\subseteq \mV_i $. Since $\mV_i \cap \mV_j \neq \emptyset$ without loss of generality we assume the set $\mV_i \cap \mV_j$ is of the form $ \{\vb^{l-1}_{i;1}, \vb^{l-1}_{i;2},\cdots, \vb^{l-1}_{i;k_i} \}$ for some $k_i \leqslant C^{l-1}_i$. Also, there exist $k_j \leqslant C^l_j$ such that 
$\vb^{l-1}_{j;k_j} = \vb^{l-1}_{i;1}$. And therefore the sequence $\vb^{l}_i$ is set to be $(\vb^{l-1}_{j;1},\vb^{l-1}_{j;2},\cdots \vb^{l-1}_{j;k_j}, \vb^{l-1}_{i;2},\vb^{l-1}_{i;3} \cdots \vb^{l-1}_{i;C^{l-1}_i} )$. \label{chain_rule_2}
\end{itemize}
Moreover, for a given node $N^{l-1}_i$ there might exist more than one node $N^{l-1}_j$ satisfying $\mA^{(l-1)}_{i,j} \geqslant 1$. 
Denote $\Lambda^{l}_i = \{ j| \mA^{(l)}_{i,j} \geqslant 1 \}$, 
then the information in each node $N^{l-1}_k$ with $k\in \Lambda^{l}_i$ is transmitted to $N^{l-1}_i$ as the above way by treating $\vb^{l}_i$ as $\vb^{l-1}_i$ each time. 
More specifically, we set initially $N^{l}_i = N^{l-1}_i$ and correspondingly $\vb^{l}_i = \vb^{l-1}_i$. Then for each $k\in \Lambda^{l}_i$ and for $\vb^{l-1}_k$ associated to $N^{l-1}_k$, we update $\vb^{(l)}_i$ as in the above two rules by setting $\vb^{l-1}_i = \vb^{l}_i$ and $\vb^{l-1}_j = \vb^{l-1}_k $.
\begin{remark}

The information propagation in this transformer satisfies the rules as we defined in section \ref{Rules_of_Inf_Prop}. 
Although we ignore the information propagated from the node $N^{l-1}_1$ by setting $\mA^{(l-1)}_{i,  1}=0$ for $l\geqslant 1$, 
there won't be any information loss. Since the first node in each layer will only contain one value which is also contained in $N^1_2$ by Rule 2.
\end{remark}

\subsection{Existence of approximating FNN and error analysis}
We find the FNN we required in three steps. 
\begin{itemize}
  \item \textbf{Step 1} Find continuous functions that decode the information;
  \item \textbf{Step 2} Extend the continous function to allow small error;
  \item \textbf{Step 3} Use universal approximation theorem to find a FNN to approximate the extended continuous functions. 
\end{itemize}

\subsection*{\textbf{Step 1}, Continuous functions on the ideal inputs}
Fix the dictionary, sequence length, depth, embeddings, and attention parameters of the construction. Let $\Sigma$ be the finite family of all admissible reasoning inputs of length $n$ over this dictionary. The index $\sigma\in\Sigma$ specifies the input, including its token choices, ordering of facts, and query. For a fixed reasoning chain, this family includes all its admissible orderings. Write
\begin{equation}\label{eq:ideal_decoder_data}
 Y_{\sigma,i}^l=\mX_{\sigma,i}^{ao(l)},\qquad
 T_{\sigma,i}^{l+1}=\mX_{\sigma,i}^{(l+1)},\qquad 0\le l<L,
\end{equation}
where $T_{\sigma,i}^{l+1}$ is the ideal binary sparse code prescribed by the information-propagation construction. These ideal states are defined before choosing any approximating FNN.

To define a function of $Y_{\sigma,i}^l$, it is necessary that equal inputs have equal targets. Lemma~\ref{distance_Lemma}, proved below directly from the attention inputs, establishes this consistency. In addition, Lemma~\ref{rmsnorm_injective} gives an inverse image under $R_N^l$ for every binary target. Thus the local pre-normalization map $g_i^l$ is well-defined on the finite set $\mathcal S_i^l=\{Y_{\sigma,i}^l:\sigma\in\Sigma\}$ by
\begin{equation}\label{continous_version}
 g_i^l(Y_{\sigma,i}^l)=(R_N^l)^{-1}(T_{\sigma,i}^{l+1}),
 \qquad R_N^l(g_i^l(Y_{\sigma,i}^l))=T_{\sigma,i}^{l+1}.
\end{equation}
A function on a finite subset of Euclidean space is continuous in the relative topology. Step 2 extends this map to neighborhoods of the ideal inputs. The map $g_i^l$ denotes the total pre-normalization target; the FNN branch must approximate $g_i^l(y)-y$ to account for the residual in \eqref{output}.

We use the following form of approximation after normalization.
\begin{lemma}\label{approximate_FNN}
Let $K\subseteq\mathbb R^p$ be nonempty and compact, and let $f:K\to\mathbb R^{d_m}$ be continuous. Fix RMSNorm with a positive stability constant and a fixed positive scalar gain. For every $\epsilon^{\prime}>0$, there exists a single-hidden-layer FNN $f^{\prime}:\mathbb R^p\to\mathbb R^{d_m}$, with a fixed continuous sigmoidal activation and sufficiently many hidden units, such that
\begin{equation}\label{eq:continuous_normalized_approximation}
 \sup_{x\in K}\|R_N(f(x))-R_N(f^{\prime}(x))\|_\infty
 <\epsilon^{\prime}.
\end{equation}
\end{lemma}
Here and in the sequel, $\|\cdot\|$ denotes the $\infty$ norm: for $v\in\mathbb R^{d_m}$, $\|v\|=\max_{1\le a\le d_m}|v_a|$. Continuity on the approximation domain is essential; a discontinuous simple function on an arbitrary compact set need not admit uniform approximation by continuous neural networks. In the construction below, the local maps are constant on finitely many separated compact sets and hence continuous on their union.

As discussed in Remark~\ref{impossible}, we use arbitrarily accurate approximations to the prescribed ideal outputs rather than assuming that an FNN can be assigned arbitrary exact outputs on neighborhoods. We next construct the compact domains on which these approximations will be used.

\subsection*{\textbf{Step 2: Extension of the local decoding maps}}
For $\sigma\in\Sigma$, position $i$, and $0\le l<L$, define the ideal input set independently of any decoding function:
\begin{equation}\label{def_of_Dset}
 D_{\sigma,i}^{l+1}=\{Y_{\sigma,i}^l\}\subseteq\mathbb R^{d_m}.
\end{equation}
We group inputs according to their desired output at the next layer.
\begin{definition}
For fixed $i$ and $l$, two inputs $\sigma,\tau\in\Sigma$ are equivalent if $T_{\sigma,i}^{l+1}=T_{\tau,i}^{l+1}$. Their equivalence classes and the family of all classes are denoted by
\begin{equation}
 [\sigma]_i^{l+1}=\{\tau\in\Sigma:T_{\tau,i}^{l+1}=T_{\sigma,i}^{l+1}\},
 \qquad E_i^{l+1}=\Sigma/{\sim_i^{l+1}}.
\end{equation}
For a class $c=[\sigma]_i^{l+1}$, let $T_{c,i}^{l+1}=T_{\sigma,i}^{l+1}$ and
\begin{equation}\label{eq:decoder_class_inputs}
 D_{c,i}^{l+1}=\bigcup_{\tau\in c}D_{\tau,i}^{l+1}.
\end{equation}
\end{definition}
The union in \eqref{eq:decoder_class_inputs} includes every input in the class, not just a representative. All these sets are finite because $\Sigma$ is finite.

\begin{lemma}\label{distance_Lemma}
For any $\sigma,\tau\in\Sigma$, fixed position $i$, and $0\le l<L$,
\begin{equation}\label{eq:local_decoder_consistency}
 Y_{\sigma,i}^l=Y_{\tau,i}^l
 \quad\Longrightarrow\quad
 T_{\sigma,i}^{l+1}=T_{\tau,i}^{l+1}.
\end{equation}
Consequently, ideal input sets belonging to different target classes are disjoint.
\end{lemma}
If $E_i^{l+1}$ has at least two classes, define
\begin{equation}\label{eq:local_decoder_separation}
 d_i^{l+1}
 =\min_{\substack{c,c'\in E_i^{l+1},\ c\ne c'\\
                    x\in D_{c,i}^{l+1},\ y\in D_{c',i}^{l+1}}}
   \|x-y\|_\infty>0.
\end{equation}
If there is only one target class, set $d_i^{l+1}=1$; no separation between different targets is needed in that case. Choose
\begin{equation}\label{eq:local_decoder_radius}
 0<\delta_i^{l+1}<\frac{d_i^{l+1}}{3},
\end{equation}
and form the compact class neighborhoods and their union:
\begin{equation}\label{eq:local_decoder_neighborhoods}
 \begin{aligned}
 K_{c,i}^l&=D_{c,i}^{l+1}
       +[-\delta_i^{l+1},\delta_i^{l+1}]^{d_m},\\
 K_i^l&=\bigcup_{c\in E_i^{l+1}}K_{c,i}^l.
 \end{aligned}
\end{equation}
Here the addition is Minkowski addition in $\mathbb R^{d_m}$. For different classes, the distance between these neighborhoods is at least $d_i^{l+1}-2\delta_i^{l+1}>0$, so they cannot intersect. Neighborhoods of centers within the same class may overlap, since they require the same target.

Define the extended pre-normalization map on $K_i^l$ by
\begin{equation}\label{Extended_f}
 \begin{aligned}
 \tilde f_i^l(y)
 &=(R_N^l)^{-1}\!\left(
     \sum_{c\in E_i^{l+1}}T_{c,i}^{l+1}\mathbf1_{K_{c,i}^l}(y)
                      \right),\qquad y\in K_i^l,\\
 &=(R_N^l)^{-1}(T_{c,i}^{l+1})
                      \qquad\text{when }y\in K_{c,i}^l.
 \end{aligned}
\end{equation}
At each point of $K_i^l$, exactly one class indicator is active. Its target is a binary vector in the range of RMSNorm. Hence \eqref{Extended_f} is well-defined, is constant on each class neighborhood, and extends $g_i^l$. The finitely many class neighborhoods are separated compact sets, so $\tilde f_i^l$ is continuous on their union. No continuity of the displayed indicator formula outside $K_i^l$ is asserted or needed; a continuous extension to the ambient space will be used for neural-network approximation.

\input{shared_ffn.tex}

\subsection*{\textbf{Step 3}, approximating FNN}
By Proposition~\ref{prop:shared_position_ffn}, choose one shared FNN $\widehat F^l$ on the compact union of the position domains, with normalized-output error less than $\min_i\eta_i^{l+1}$. In the formulas below that write $R_N^l(\hat f_i^l(y))$ without displaying the residual separately, $\hat f_i^l(y)$ denotes the total pre-normalization map $\widehat F^l(y)+y$ restricted to the position-$i$ domain. Thus the actual FNN branch is the same $\widehat F^l$ at every position. With the neighborhoods shrunk as above, we have
\begin{equation}
  \sup\limits_{\mX \in  D^{l+1}_{\sigma,i}+[-\delta_{i}^{l+1},\delta_{i}^{l+1}]^{d_m}}\|R^{l}_N(\hat{f}_{i}^{l}(\mX))-R^{l}_N(\tilde{f}_{i}^{l}(\mX))\|_\infty<\eta_{i}^{(l+1)}.
\end{equation}
The shared FNN branch $\widehat F^l$ therefore realizes the required position-wise approximations. In fact we have the following proposition.
\begin{proposition}\label{prop:attention_perturbation}
With the local decoding neighborhoods constructed above, for any prescribed output tolerance the shared FNNs can be chosen so that their normalized outputs approximate the ideal hidden states at every position and layer within that tolerance. This remains valid when the FNN inputs include the errors from the preceding layers.
\end{proposition}

\begin{proof}
We estimate the errors in the attention scores, softmax weights, and attention-residual inputs explicitly. All vector norms below are the $\infty$ norm. The approximation tolerance $\epsilon$ used here is distinct from the fixed positive stability constant in RMSNorm. When $L=1$, only the first-layer approximation argument at the end of the proof is needed; the estimates below concern $L\ge2$.

For $1\le l<L$, write the perturbed hidden states as
\begin{equation}\label{eq:hidden_state_perturbation}
 \bar\mX_{\sigma,i}^{(l)}
 =\mX_{\sigma,i}^{(l)}+\epsilon_i^{(l)}\vr_i^{(l)},
 \qquad 0\le\epsilon_i^{(l)}\le\epsilon,\quad
 \|\vr_i^{(l)}\|_\infty\le1.
\end{equation}
Put $K=(n+1)3^L$ and $H=\sum_{m=-K}^{-1}\mR^m$. For brevity, write $X_i=\mX_{\sigma,i}^{(l)}$ in the score calculation. Only unmasked entries $1\le j\le i\le n$ are considered; masked positions always have zero attention weight. Expanding the perturbed score gives all three error terms:
\begin{equation}\label{eq:attention_perturbation_expansion}
 \begin{aligned}
 \bar A_{i,j}^{(l)}
 &=(X_i+\epsilon_i^{(l)}\vr_i^{(l)})
 H(X_j+\epsilon_j^{(l)}\vr_j^{(l)})^{\mathsf T}\\
 &=A_{i,j}^{(l)}
 +\epsilon_i^{(l)}\vr_i^{(l)}HX_j^{\mathsf T}
 +\epsilon_j^{(l)}X_iH(\vr_j^{(l)})^{\mathsf T}\\
 &\quad+\epsilon_i^{(l)}\epsilon_j^{(l)}
 \vr_i^{(l)}H(\vr_j^{(l)})^{\mathsf T}.
 \end{aligned}
\end{equation}
Every $\mR^m$ is a permutation matrix, so $H$ has nonnegative entries and $\sum_{a,b}H_{ab}=d_mK$. The ideal codes satisfy $\|X_i\|_\infty\le1$. Thus the first two error terms are of first order, and
\begin{equation}\label{eq:attention_perturbation_bound}
 \begin{aligned}
 |\bar A_{i,j}^{(l)}-A_{i,j}^{(l)}|
 &\le d_mK\bigl(\epsilon_i^{(l)}+\epsilon_j^{(l)}
                  +\epsilon_i^{(l)}\epsilon_j^{(l)}\bigr)\\
 &\le d_mK(2\epsilon+\epsilon^2)\le\eta_0,
 \end{aligned}
\end{equation}
provided $0<\epsilon\le\min\{1,\eta_0/(3d_mK)\}$. In particular, the score error still tends to zero as $\epsilon\to0$.

Define the normalizing sums and causal attention weights by
\begin{equation}\label{eq:perturbed_attention_weights}
 \begin{aligned}
 Z_i^{(l)}&=\sum_{t=1}^{i}e^{A_{i,t}^{(l)}},&
 \bar Z_i^{(l)}&=\sum_{t=1}^{i}e^{\bar A_{i,t}^{(l)}},\\
 w_{ij}^{(l)}&=\frac{e^{A_{i,j}^{(l)}}}{Z_i^{(l)}},&
 \bar w_{ij}^{(l)}&=\frac{e^{\bar A_{i,j}^{(l)}}}{\bar Z_i^{(l)}}.
 \end{aligned}
\end{equation}
Since the value projection is the identity for these layers, the perturbed attention-residual input is
\begin{equation}
 \begin{aligned}
 \bar\mX_{\sigma,i}^{ao(l)}
 &=\sum_{j=1}^{i}\bar w_{ij}^{(l)}\bar\mX_{\sigma,j}^{(l)}
   +\bar\mX_{\sigma,i}^{(l)}\\
 &=\sum_{j=1}^{i}\bar w_{ij}^{(l)}
   (\mX_{\sigma,j}^{(l)}+\epsilon_j^{(l)}\vr_j^{(l)})
   +\mX_{\sigma,i}^{(l)}+\epsilon_i^{(l)}\vr_i^{(l)}.
 \end{aligned}
\end{equation}
Subtracting the ideal attention-residual input yields
\begin{equation}\label{Xao_difference}
 \begin{aligned}
 \bar\mX_{\sigma,i}^{ao(l)}-\mX_{\sigma,i}^{ao(l)}
 &=\sum_{j=1}^{i}(\bar w_{ij}^{(l)}-w_{ij}^{(l)})\mX_{\sigma,j}^{(l)}\\
 &\quad+\sum_{j=1}^{i}\bar w_{ij}^{(l)}\epsilon_j^{(l)}\vr_j^{(l)}
       +\epsilon_i^{(l)}\vr_i^{(l)}.
 \end{aligned}
\end{equation}
The ideal hidden states at distinct positions have disjoint supports and coefficients zero or one, as established in Proposition~\ref{prop:shared_position_ffn}. Consequently,
\begin{equation}
 \|\bar\mX_{\sigma,i}^{ao(l)}-\mX_{\sigma,i}^{ao(l)}\|_\infty
 \le\mathrm I+\mathrm {II},
\end{equation}
where
\begin{align}
 \mathrm I&=\max_{l,\,1\le j\le i\le n}|\bar w_{ij}^{(l)}-w_{ij}^{(l)}|,\\
 \mathrm {II}&=\left\|\sum_{j=1}^{i}\bar w_{ij}^{(l)}\epsilon_j^{(l)}\vr_j^{(l)}
                +\epsilon_i^{(l)}\vr_i^{(l)}\right\|_\infty.
\end{align}

Choose $0<\eta_0\le\log 2$, so that $|e^x-1|\le2|x|$ for $|x|\le\eta_0$. Set $M=d_mK+1$, which bounds the absolute values of both the ideal and perturbed unmasked scores under \eqref{eq:attention_perturbation_bound}. Hence $\bar Z_i^{(l)}\ge e^{-M}$ and
$|e^{\bar A_{i,j}^{(l)}}-e^{A_{i,j}^{(l)}}|\le2\eta_0e^M$.
Keeping the two contributions from the numerator and denominator separate, we obtain
\begin{align}
 \begin{aligned}\label{partial_1}
 \mathrm I
 &\le\max_{l,\,j\le i}\left\{
 \frac{|e^{\bar A_{i,j}^{(l)}}-e^{A_{i,j}^{(l)}}|}{\bar Z_i^{(l)}}
 +\frac{e^{A_{i,j}^{(l)}}}{Z_i^{(l)}}
  \frac{|Z_i^{(l)}-\bar Z_i^{(l)}|}{\bar Z_i^{(l)}}\right\}\\
 &\le2\eta_0e^{2M}+2n\eta_0e^{2M}
 \le4n\eta_0e^{2M},
 \end{aligned}\\
 \mathrm {II}&\le\sum_{j=1}^{i}\bar w_{ij}^{(l)}\epsilon+\epsilon
             =2\epsilon\le(n+1)\epsilon.\label{partial_2}
\end{align}
Combining \eqref{Xao_difference}, \eqref{partial_1}, and \eqref{partial_2} gives the attention-residual error bound:
\begin{equation}\label{eq:attention_residual_error_bound}
 \|\bar\mX_{\sigma,i}^{ao(l)}-\mX_{\sigma,i}^{ao(l)}\|_\infty
 \le4n\eta_0e^{2M}+(n+1)\epsilon.
\end{equation}

Let $\delta_* =\min_{i,\,0\le l<L}\delta_i^{l+1}>0$ be a common lower bound for the admissible neighborhood radii after the shrinking in Proposition~\ref{prop:shared_position_ffn}. A sufficient choice is
\begin{equation}\label{eq:attention_perturbation_tolerances}
 \begin{aligned}
 0<\eta_0&<\min\left\{\log 2,\frac{\delta_*}{8ne^{2M}}\right\},\\
 0<\epsilon&<\min\left\{1,\frac{\eta_0}{3d_mK},
                         \frac{\delta_*}{2(n+1)}\right\}.
 \end{aligned}
\end{equation}
The tolerance $\epsilon$ may also be taken smaller than any prescribed output tolerance. These choices imply both \eqref{eq:attention_perturbation_bound} and
\begin{equation}
 \|\bar\mX_{\sigma,i}^{ao(l)}-\mX_{\sigma,i}^{ao(l)}\|_\infty
 <\delta_*\le\delta_i^{l+1}.
\end{equation}
Thus the perturbed FNN input remains in the same admissible decoding neighborhood as its ideal counterpart. Choose the normalized-output approximation errors in Step 3 to satisfy $\eta_i^{l+1}<\epsilon$ for all positions and layers. Since the extended local decoder has the same ideal target on this neighborhood, the shared FNN gives
\begin{equation}\label{eq:ffn_restores_accuracy}
 R_N^l\bigl(\hat f_i^l(\bar\mX_{\sigma,i}^{ao(l)})\bigr)
 \in\{\mX_{\sigma,i}^{(l+1)}\}+[-\epsilon,\epsilon]^{d_m}.
\end{equation}
As specified in Step 3, $\hat f_i^l(y)=\widehat F^l(y)+y$ includes the FNN residual. The first layer starts from exact embeddings and can attain the same error tolerance by uniform approximation. Induction over the remaining finitely many layers now applies \eqref{eq:attention_residual_error_bound} and \eqref{eq:ffn_restores_accuracy} at each step. The two first-order terms change the required approximation accuracy, but do not prevent it from being chosen sufficiently small.
\end{proof}

\begin{lemma}\label{rmsnorm_injective}
For $x\in\mathbb R^{d_m}$, define its mean square by
\begin{equation}\label{rmsnorm_mean_square}
 q(x)=\frac{1}{d_m}\sum_{i=1}^{d_m}(x^i)^2.
\end{equation}
RMSNorm with a scalar gain $\alpha>0$, zero bias, and a stability constant $\epsilon>0$ is
\begin{equation}\label{rmsnorm_definition}
 R_N(x)=\alpha\frac{x}{\sqrt{q(x)+\epsilon}}.
\end{equation}
This map is injective. Its range consists of the vectors $z$ satisfying $q(z)<\alpha^2$, and its inverse on that range is
\begin{equation}\label{rmsnorm_inverse}
 R_N^{-1}(z)=\sqrt{\frac{\epsilon}{\alpha^2-q(z)}}\,z.
\end{equation}
In the construction we choose $\alpha>1$, so that every vector in $\{0,1\}^{d_m}$ belongs to this range.
\end{lemma}
\begin{proof}[\textbf{Proof of Lemma \ref{rmsnorm_injective}}]
The denominator in \eqref{rmsnorm_definition} is strictly positive for every input. Suppose that $R_N(x_1)=R_N(x_2)$. Written coordinate by coordinate, this means
\begin{equation}\label{rmsnorm_injective_temp1}
 \alpha\frac{x_1^i}{\sqrt{q(x_1)+\epsilon}}
 =\alpha\frac{x_2^i}{\sqrt{q(x_2)+\epsilon}},
 \qquad 1\le i\le d_m.
\end{equation}
If $x_2=0$, these equations imply $x_1=0$ as well. For all other inputs, we can cancel the nonzero scalar $\alpha$ and define the positive number
\begin{equation}\label{rmsnorm_temp2}
 c=\frac{\sqrt{q(x_1)+\epsilon}}{\sqrt{q(x_2)+\epsilon}}>0.
\end{equation}
Equation~\eqref{rmsnorm_injective_temp1} then gives
\begin{equation}\label{rmsnorm_temp3}
 x_1^i=cx_2^i\quad(1\le i\le d_m),\qquad x_1=cx_2.
\end{equation}
Taking the mean square of both sides yields
\begin{equation}
 q(x_1)=\frac1{d_m}\sum_{i=1}^{d_m}(cx_2^i)^2=c^2q(x_2).
\end{equation}
On the other hand, squaring \eqref{rmsnorm_temp2} and clearing the positive denominator gives
\begin{equation}
 c^2\bigl(q(x_2)+\epsilon\bigr)=q(x_1)+\epsilon.
\end{equation}
Substituting the preceding mean-square identity, we obtain
\begin{equation}
 c^2q(x_2)+c^2\epsilon=c^2q(x_2)+\epsilon,
 \qquad (c^2-1)\epsilon=0.
\end{equation}
Since $\epsilon>0$ and $c>0$, it follows that $c=1$. Equation~\eqref{rmsnorm_temp3} therefore gives $x_1=x_2$, proving injectivity. The positive stability constant is essential to this argument.

To determine the range, take the mean square of the normalized output:
\begin{equation}
 q(R_N(x))=\frac{\alpha^2q(x)}{q(x)+\epsilon}<\alpha^2.
\end{equation}
Conversely, suppose that $q(z)<\alpha^2$ and put
$u=\sqrt{\epsilon/(\alpha^2-q(z))}\,z$. Direct calculation gives
\begin{equation}
 q(u)+\epsilon
 =\frac{\epsilon q(z)}{\alpha^2-q(z)}+\epsilon
 =\frac{\epsilon\alpha^2}{\alpha^2-q(z)}.
\end{equation}
Using $\alpha>0$ in \eqref{rmsnorm_definition}, we thus have
\begin{equation}
 R_N(u)=
 \frac{\alpha\sqrt{\epsilon/(\alpha^2-q(z))}\,z}
 {\sqrt{\epsilon\alpha^2/(\alpha^2-q(z))}}=z.
\end{equation}
This proves the range statement and the inverse formula. Finally, a binary vector satisfies $q(z)\le1$; choosing $\alpha>1$ therefore supplies a preimage for every binary target used by the construction.
\end{proof}

\begin{proof}[\textbf{Proof of Lemma \ref{distance_Lemma}}]
We show directly that the attention-residual input determines the prescribed next-layer target. This argument uses only the ideal sparse codes and the fixed attention parameters, before any approximating FNN is chosen.

\textbf{First-layer inputs ($l=0$).}
Let $t_j$ denote the dictionary index of the token at input position $j$, so its embedding is $e_{k_{t_j}}$. By the embedding spacing, the token coordinates $k_t$, shifted token coordinates $k_t-1$, and positional coordinates do not collide. The residual therefore identifies the original token at position $i$ through
\begin{equation}\label{eq:first_layer_original_token}
 (Y_{\sigma,i}^{0})_{k_t}=\mathbf1_{\{t=t_i\}}.
\end{equation}
For odd $i$, the first-layer target contains only this original token, with the prescribed position shift (including the special shift at $i=1$). Thus equality of the FNN inputs implies equality of their targets in this case.

For even $i$, the target encodes the ordered pair $(t_{i-1},t_i)$, so we must also recover the preceding token. Under the unscaled-score convention stated in the architecture section, the only causal score equal to $1$ in this row is the score at $j=i-1$; all other scores are $0$. The conventional factor $1/\sqrt{d_k}$ can be retained explicitly by multiplying the displayed query matrices by $\sqrt{d_k}$, giving the same effective scores. The softmax denominator is consequently $Z_i=i-1+e$, independent of the token choices. If
\begin{equation}
 N_t=|\{j:1\le j\le i,\ t_j=t\}|,
\end{equation}
then the coefficient at the shifted token coordinate satisfies
\begin{equation}\label{eq:first_layer_predecessor_token}
 Z_i(Y_{\sigma,i}^{0})_{k_t-1}
   =N_t+(e-1)\mathbf1_{\{t=t_{i-1}\}}.
\end{equation}
The counts $N_t$ are integers. Since $e-1$ is irrational, the unique noninteger value in \eqref{eq:first_layer_predecessor_token}, as $t$ ranges over the dictionary, identifies $t_{i-1}$. Equivalently, if two inputs had different predecessor tokens but identical attention-residual vectors, equality at one of these coordinates would equate an integer with an integer plus $e-1$, a contradiction. Together with \eqref{eq:first_layer_original_token}, this determines the entire first-layer target, even when token occurrences in the prefix repeat.

\textbf{Subsequent inputs ($1\le l<L$).}
Use the coordinate blocks $\mathcal B_j$ in \eqref{eq:shared_coordinate_blocks}. Their separation follows solely from the fixed embedding spacing and the rank-offset bound, as verified in the geometric part of the proof of Proposition~\ref{prop:shared_position_ffn}; it does not depend on the existence of the local maps or of an FNN.

All causal softmax weights are strictly positive. The ideal codes are binary and their supports lie in disjoint position blocks. Consequently, for every source position $j\le i$, the input vector determines its full ideal code through
\begin{equation}\label{eq:recover_prefix_codes}
 (\mX_{\sigma,j}^{(l)})_a
 =\mathbf1_{\{a\in\mathcal B_j,\ (Y_{\sigma,i}^{l})_a>0\}},
 \qquad 1\le a\le d_m.
\end{equation}
This formula applies to the receiver as well, since its nonzero coefficients in $Y_{\sigma,i}^{l}$ equal $1+w_{ii}^{(l)}>0$. At position $1$ it recovers the sentinel code in the same way.

The coordinates of each recovered regular code specify its tokens and their offsets relative to the original token: a coordinate $k_t-j3^L-r$ identifies token $t$ and offset $r$. The token at offset zero is the original token, and sorting offsets gives the encoded chain order. The sentinel code identifies its single token. Hence the recovered prefix codes determine which earlier chains match the receiver, their relative order, and the prescribed next-layer target. More explicitly, each matching chain shares a token with the receiver; aligning their offsets at that token fixes its placement in the merged chain. The admissible reasoning inputs use a common chain, so these alignments agree on shared tokens. The union with the receiver then determines the target code, anchored at the receiver's original token.

Thus if $Y_{\sigma,i}^l=Y_{\tau,i}^l$, \eqref{eq:recover_prefix_codes} gives the same ordered prefix codes for both inputs and therefore the same next-layer target. This proves \eqref{eq:local_decoder_consistency} for every layer. Different targets must consequently have different ideal FNN inputs; finiteness yields the positive distances used in \eqref{eq:local_decoder_separation}.
\end{proof}

\begin{proof}[\textbf{Proof of Lemma \ref{approximate_FNN}}]

Since $f$ is continuous on the compact set $K$, it is bounded there. Extend each coordinate continuously from the closed set $K$ to $\mathbb R^p$. Theorem~\ref{thm:universal} applied to this extension gives, for every $\epsilon_0>0$, a single-hidden-layer FNN $f^{\prime}$ such that
\begin{equation}\label{eq:pre_normalization_approximation}
 \sup_{x\in K}\|f(x)-f^{\prime}(x)\|_\infty<\epsilon_0.
\end{equation}
We now bound the change produced by RMSNorm. The norm in this argument is the maximum norm, as fixed above. Put $M=\sup_{u\in K}\|f(u)\|_\infty$ and restrict $0<\epsilon_0\le1$. For $x=f(u)$ and $y=f^{\prime}(u)$, we then have
\begin{equation}
 \|x-y\|_\infty<\epsilon_0,\qquad
 \|x\|_\infty\le M,\qquad \|y\|_\infty\le M+1.
\end{equation}
Writing $a=\sqrt{q(x)+\epsilon}$ and $b=\sqrt{q(y)+\epsilon}$, we have $a,b\ge\sqrt\epsilon$. The normalized difference can be decomposed without subtracting a coordinate mean:
\begin{equation}
 \begin{aligned}
 R_N(x)-R_N(y)
 &=\alpha\left(\frac{x}{a}-\frac{y}{b}\right)\\
 &=\alpha\left(\frac{x-y}{a}+y\left(\frac1a-\frac1b\right)\right).
 \end{aligned}
\end{equation}
Consequently,
\begin{equation}
 \|R_N(x)-R_N(y)\|_\infty
 \le |\alpha|\left(
 \frac{\|x-y\|_\infty}{\sqrt\epsilon}
 +\|y\|_\infty\frac{|a-b|}{\epsilon}\right).
\end{equation}
To estimate the second term, use the definition of the mean square in \eqref{rmsnorm_mean_square}:
\begin{equation}
 \begin{aligned}
 |q(x)-q(y)|
 &\le\frac1{d_m}\sum_{i=1}^{d_m}|x^i-y^i|\,|x^i+y^i|\\
 &\le \|x-y\|_\infty(\|x\|_\infty+\|y\|_\infty)\\
 &\le \epsilon_0(2M+1).
 \end{aligned}
\end{equation}
The difference between the two square roots is obtained by multiplying by their sum:
\begin{equation}
 \begin{aligned}
 |a-b|
 &=\frac{|q(x)-q(y)|}{\sqrt{q(x)+\epsilon}+\sqrt{q(y)+\epsilon}}\\
 &\le\frac{\epsilon_0(2M+1)}{2\sqrt\epsilon}.
 \end{aligned}
\end{equation}
Combining these inequalities gives the explicit bound
\begin{equation}
 \begin{aligned}
 \|R_N(x)-R_N(y)\|_\infty
 &\le |\alpha|\epsilon_0
 \left(\frac1{\sqrt\epsilon}
 +\frac{(M+1)(2M+1)}{2\epsilon^{3/2}}\right)\\
 &=C_{M,\alpha,\epsilon}\epsilon_0.
 \end{aligned}
\end{equation}
Here $C_{M,\alpha,\epsilon}$ is finite and does not depend on the point $u$ or on the approximating network. Given the desired tolerance $\epsilon^{\prime}>0$, choose
\begin{equation}
 0<\epsilon_0<\min\{1,\epsilon^{\prime}/C_{M,\alpha,\epsilon}\}.
\end{equation}
The preceding uniform approximation then yields
\begin{equation*}
 \sup_{u\in K}\|R_N(f(u))-R_N(f^{\prime}(u))\|_\infty
 \le C_{M,\alpha,\epsilon}\epsilon_0<\epsilon^{\prime}.
\end{equation*}

\end{proof}

\begin{theorem}[Universal Approximation Theorem\citep{cybenko1989approximation}]\label{thm:universal}
Fix a continuous sigmoidal hidden activation. For any continuous function $ f: \mathbb{R}^d \to \mathbb{R}^{n} $, compact set $K\subseteq\mathbb R^d$, and error tolerance $ \epsilon > 0 $, there exists a single-hidden-layer neural network $ f_\theta $ with biases, a linear output layer, and sufficiently many hidden units, such that:

\begin{equation*}
\sup\limits_{x \in K} ||f(x) - f_\theta(x)||_{\infty} < \epsilon,  
\end{equation*}

where $ K \subseteq \mathbb{R}^d $ is an arbitrary compact set.
\end{theorem}

\section{Details of the experiment} \label{appendix: experiment settings}
\subsection{Dataset}

We require reasoning sequence $(x_{i})_{1\le i\le 2s}$ of the training set satisfy the following condition.

\begin{align}
    x_{2i}-x_{2i-1} \bmod 5 \in\{0,1,4\} 
\end{align}

The sequence of the test set satisfy:

\begin{align}
    x_{2i}-x_{2i-1} \bmod 5 \in\{2,3\} 
\end{align}

The values of each token range from 20 to 100,i.e.,$x_{i}\in[20,100]$. 

\subsection{Hyperparameters}

In this section, the fixed and tunable hyperparameters employed in the model are outlined.

The fixed hyperparameters are as follows. Transformer architecture uses one attention head per layer. The dataset is partitioned into a training set comprising 90\% of the data and a test set comprising the remaining 10\%. Training is conducted over 2000 epochs. For embedding initialization, use $N(0,1)$; for all other parameters, use Kaiming uniform initialization with a power of 1 for small initialization. A weight decay of 0.1 is applied. The dimension of the model $d_m$  is set equal to the key dimension $d_k$. The feed-forward network dimension $d_{feedforward}$ is set to 1200.
The AdamW optimizer ($\beta_1 = 0.9, \beta_2 = 0.999$) was employed, together with a gradual warmup scheduler followed by cosine annealing learning rate decay.

\begin{table}[H]
    \centering
    \caption{}
    \label{tab:dataset_sizes}
    \begin{tabular}{lccc} %
    \toprule
    the number of reasoning steps        & 3   & 4   & 5      \\
    \midrule
    the size of datasets & 1200000  & 6000000  & 30000000     \\
    \bottomrule
    \end{tabular}
\end{table}

\subsection{Lower Bound: Experimental Results of a 3-Layer Transformer on 3-Step Reasoning}\label{app:three_step_results}
\label{3-step reasoning}

We train a 3-layer Transformer on sequences of length 21, using a batch size of 1000 and a learning rate of $5e-5$, to perform 3-step reasoning. Fig.~\ref{fig:all_dms_horizontal2} shows our training results. As the hidden dimension $d_m$ increases, the model generalizes more easily during training.

\begin{figure}[tbp]
    \centering
    \includegraphics[width=0.32\linewidth]{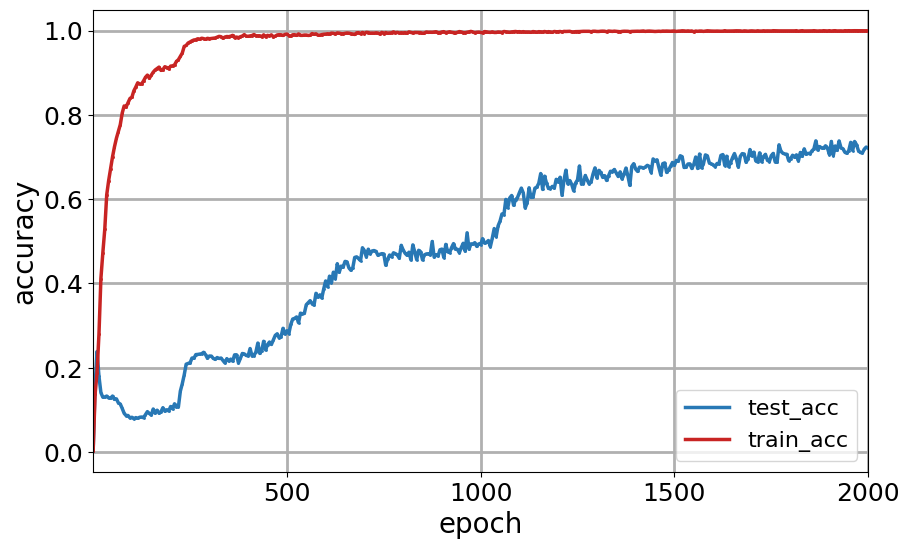}%
    \hfill%
    \includegraphics[width=0.32\linewidth]{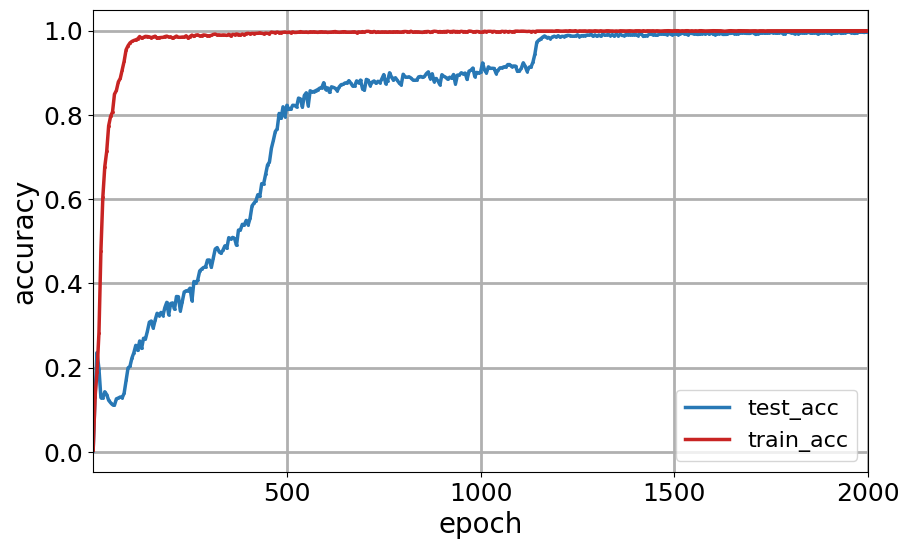}%
    \hfill%
    \includegraphics[width=0.32\linewidth]{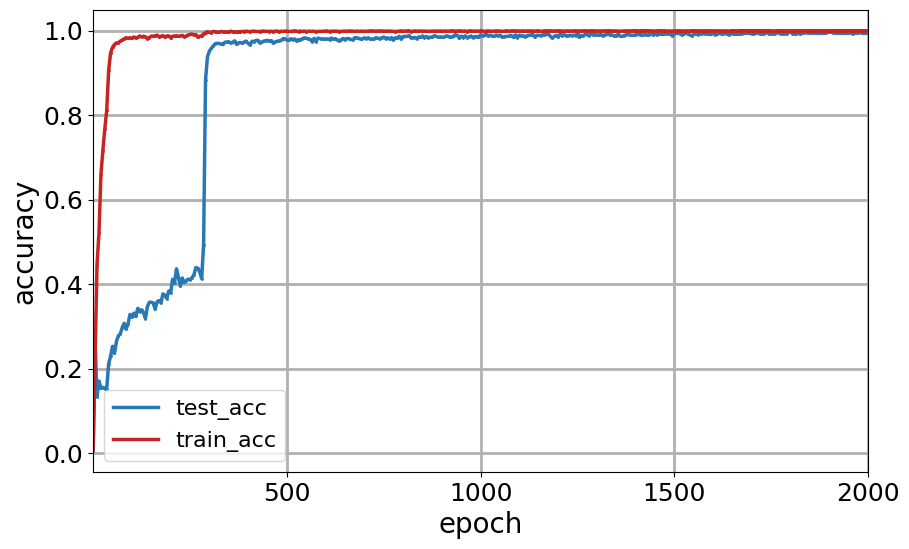}
    
    \medskip 
    \begin{minipage}{0.32\linewidth}
        \centering (a) $d_m = 256$
    \end{minipage}%
    \hfill%
    \begin{minipage}{0.32\linewidth}
        \centering (b) $d_m = 512$
    \end{minipage}%
    \hfill%
    \begin{minipage}{0.32\linewidth}
        \centering (c) $d_m = 1024$
    \end{minipage}
    
    \caption{Training curves of the 3-layer Transformer on 3-step reasoning under different hidden dimensions.}
    \label{fig:all_dms_horizontal2}
\end{figure}

\subsection{Lower Bound: The size of \texorpdfstring{$d_{m}$}{dm} required by a 3-Layer Transformer on 3-Step Reasoning}

We observe that the model dimension $d_m$ in our construction is notably large and a higher hidden dimension aids in storing more intermediate information. To validate these findings, we investigate whether a large $d_m$ is indeed necessary. Our experiments confirm that a sufficiently large $d_m$ is critical for achieving optimal performance. The relationship between the model dimension $d_{m}$ and test accuracy is summarized in the table below.

\begin{table}[H]
    \centering
    \caption{Relationship between $d_{m}$ and test accuracy}
    \label{tab:dimension_accuracy}
    \begin{tabular}{cccccc} %
    \toprule
    $d_{m}$         & 64   & 128   & 256   & 512    & 1024   \\
    \midrule
    Test accuracy (\%) & 9.6  & 11.9  & 72.2  & 99.6   & 99.6   \\
    \bottomrule
    \end{tabular}\label{tab:exp1}
\end{table}

As shown in Table \ref{tab:exp1}, test accuracy increases monotonically with dimension $d_ m$, eventually approaching $100\%$. It confirms that a large $d_{m}$ is indeed necessary.

\subsection{Upper Bound: Experimental Results of a 3-Layer Transformer on 4-Step Reasoning}\label{app:four_step_results}

We train a 3-layer Transformer with $d_m=1024$ on sequences of length 25, using a batch size of 1000 and a learning rate of $5e-5$, to perform 4-step reasoning. Fig.~\ref{fig:my_label} shows our training results. 

    \begin{figure}[tbp] 
        \centering 
        \includegraphics[width=0.5\textwidth]{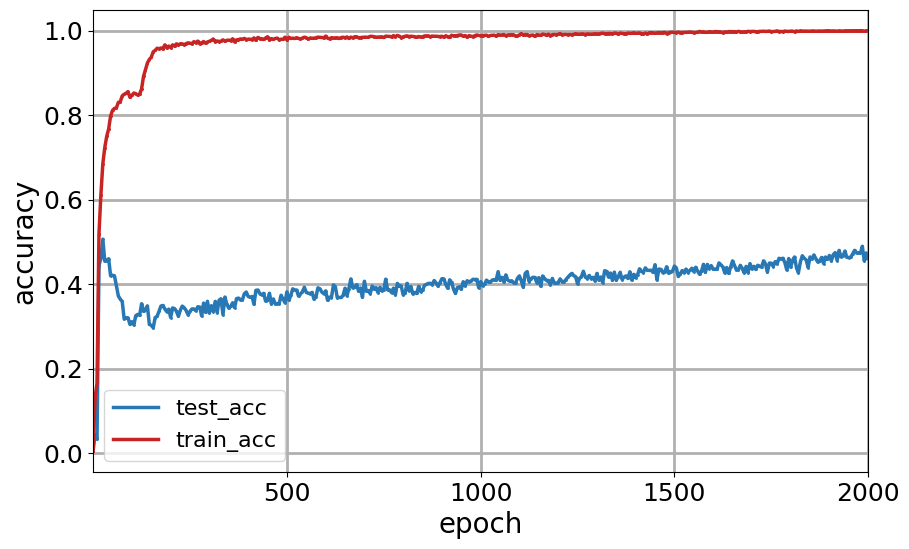} 
        \caption{Training curves of the 3-layer Transformer on 4-step reasoning}
        \label{fig:my_label} 
    \end{figure}

From Fig.~\ref{fig:my_label}, we observe that the test accuracy remains around 35\%--40\%. In fact, this is fully consistent with our theoretical expectations. Our theory indicates that a Transformer exhibits generalization only on data that satisfy a specific order structure.

As shown in Fig.~\ref{fig:canonical_appendix}, we draw out 1000 sequences from canonical sequences and non-canonical sequences to calculate accuracy. And we found that canonical sequences exhibit consistently high accuracy, while non-canonical sequences display substantially lower accuracy, thereby providing empirical support for our theoretical predictions.

\begin{figure}[tbp] 
    \centering 
    \includegraphics[width=0.6\textwidth]{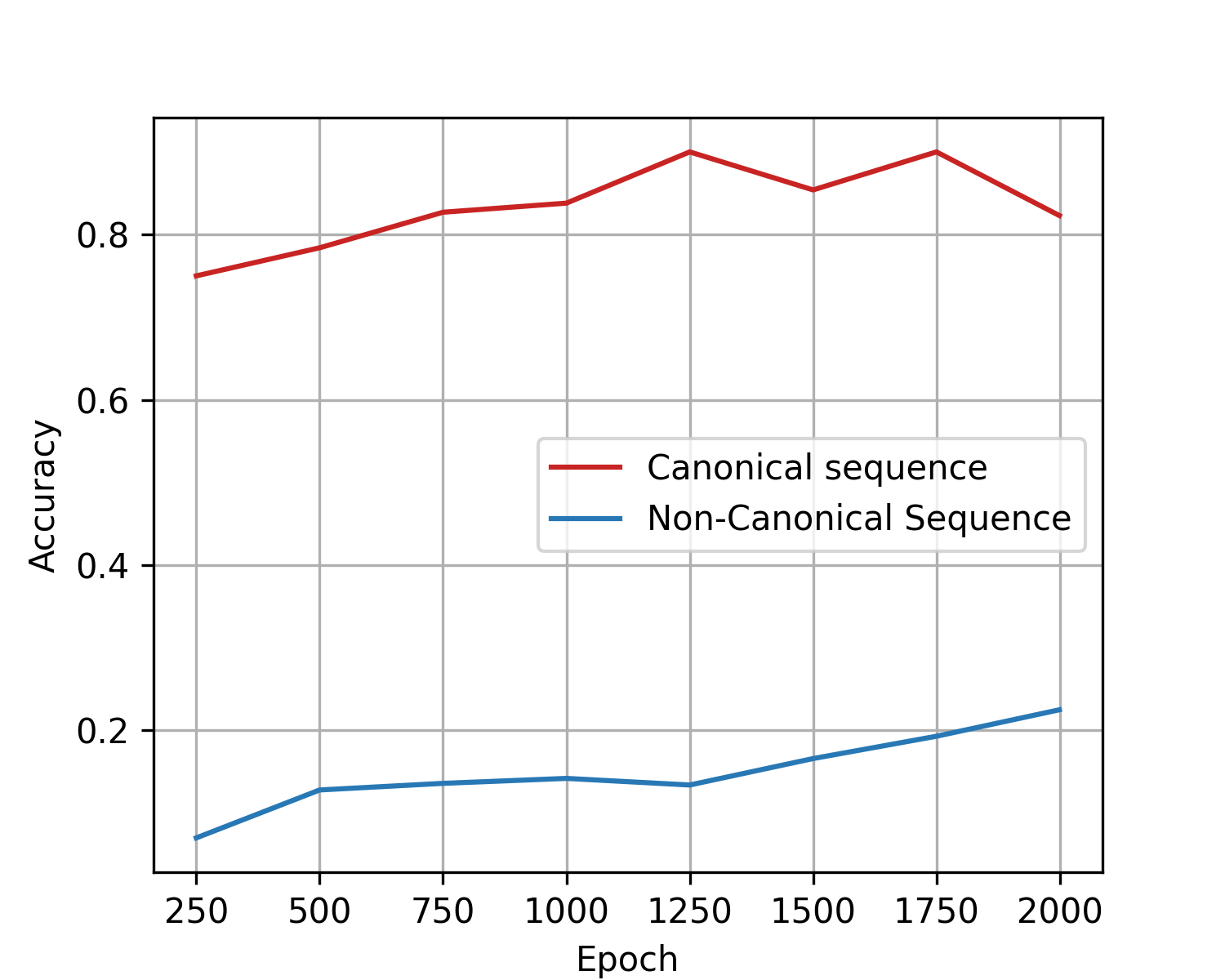} 
    \caption{Test accuracy of sentences with different sequence relationships}
    \label{fig:canonical_appendix} 
\end{figure}

\subsection{Causal intervention experiment}\label{app:causal_intervention}
\label{appendix: two order2}

In this section, we investigate whether transformer is capable of genuine reasoning or merely memorizes the answers, under the settings of 4-step and 5-step reasoning. We then describe the experimental methodology employed to obtain the results.

First, a sequence that can be answered correctly will be selected. Subsequently, a specific attention line or residual connection is masked. If transformer produces an incorrect output after the masking of a particular attention line or residual connection, that line will be marked in blue. If the model’s output remains correct, the line will left unchanged. The resulting attention graph retains only those connections that critically influence the outcome. This approach allows for conclusions to be drawn regarding whether the model has learned to perform reasoning.

To further support the validity of our theoretical analysis, we provide two additional contrastive plots (Fig.\ref{fig:contrast}). Both plots use inputs containing the same reasoning chain, except that the order of the chain elements is arranged differently. According to our theory, one particular ordering enables the model to aggregate sufficient information to arrive at the correct answer, whereas the alternative ordering prevents the model from collecting enough information to respond correctly. The experimental results align perfectly with these theoretical predictions: when (67, 74) is placed after (74, 42) and (59, 67), the critical attention path propagates exactly along the route predicted by our theory. However, if we change the relative order of these three arrays so that (67, 74) is no longer positioned after (74, 42) and (59, 67), the model’s key attention patterns become chaotic and disorganized.

\begin{figure}[tbp] 
    \centering 
    \includegraphics[width=0.75\textwidth]{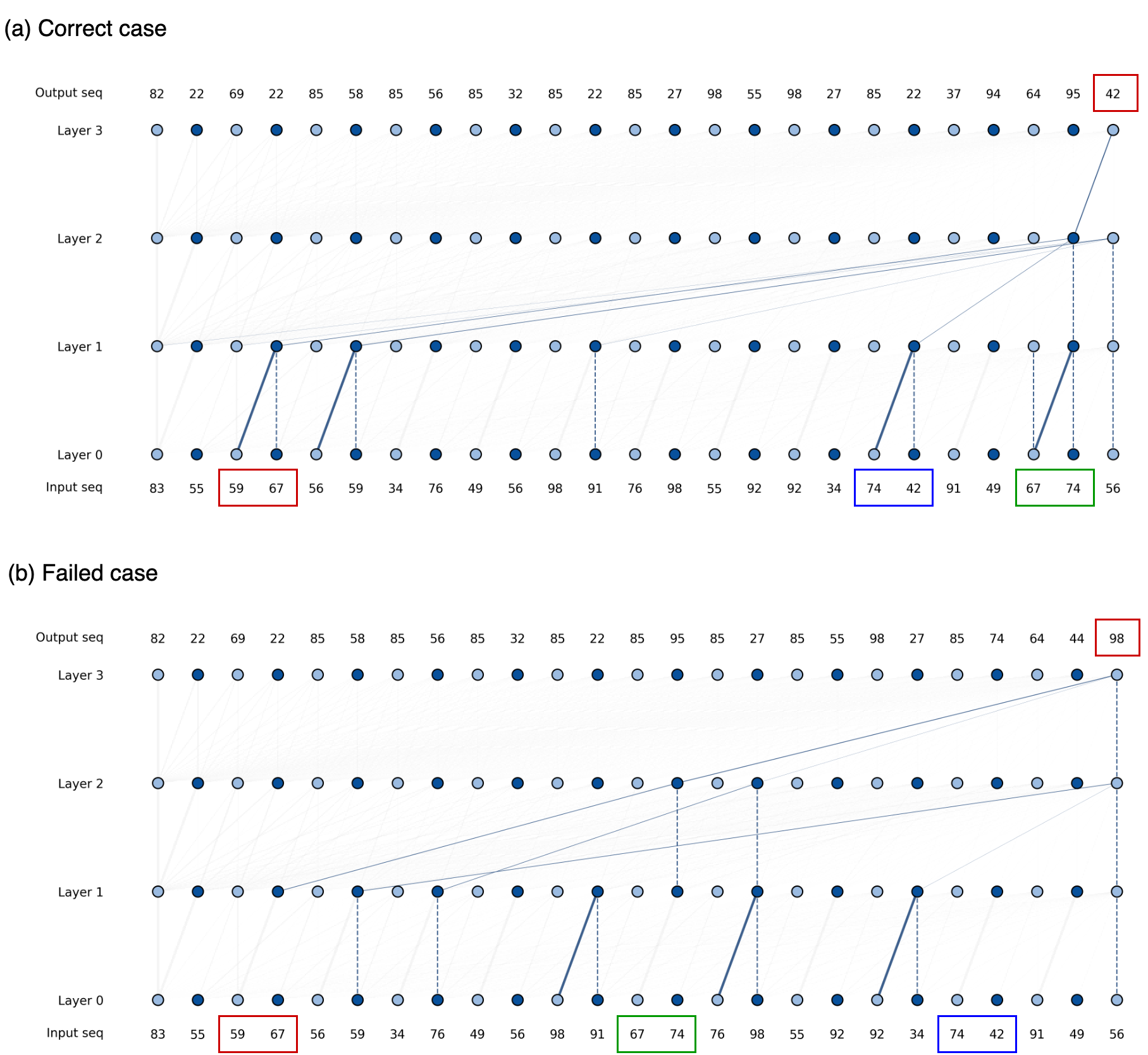} 
    \caption{(a) Attention propagation under the valid ordering predicted by our theory. When (67, 74) appears after both (74, 42) and (59, 67), the model is able to consolidate information and the critical attention route follows exactly the theoretically derived path.(b) Attention breakdown under an invalid ordering. If the relative order of the three arrays is altered such that (67, 74) is no longer placed after the other two, the model fails to aggregate sufficient information, resulting in fragmented and chaotic attention patterns} 
    \label{fig:contrast} 
\end{figure}

\subsection{Experimental Results of a 2-Layer Transformer}\label{app:two_layer_results}
\label{appendix:2-layer}

\begin{table}[H]
    \centering
    \caption{}
    \label{table3}
    \begin{tabular}{lcc} %
    \toprule
    the number of reasoning steps        & 1   & 2         \\
    \midrule
    test accuracy(\%) & 100  & 8.4     \\
    \bottomrule
    \end{tabular}
\end{table}

The experiment involves training a two-layer, single-headed transformer model with the following hyperparameter configuration: a model dimension ($d_{m}$) and key/value dimension ($d_{k}$) of 1024, a feedforward dimension ($d_{ffn}$) of 1200, a sequence length of 25, and training for 2000 epochs.

According to Corollary \ref{Corollary_Transformer}, a 2-layer transformer can theoretically solve at most one step of reasoning. The empirical results presented in Table \ref{table3} support the theoretical predictions: The model achieves complete accuracy on the one-step reasoning test set, which verifies the theoretical prediction for one-step reasoning. And a low test accuracy of $8.4\%$ is observed on the two-step reasoning test set, which indicates that the two-layer transformer's performance is not significantly better than random guessing, thereby validating the theoretical limitation for two-step reasoning.

\subsection{Guess about \texorpdfstring{$d_m$}{dm}}

Based on the aforementioned experiments, it can be observed that training a model capable of parallel reasoning—where the number of reasoning steps exceeds the depth of the transformer (i.e., number of layers minus one)—requires a substantially large model dimension $d_m$. It is therefore hypothesized that for string reasoning, wherein the number of reasoning steps equals the depth of the transformer (layers minus one), a significantly smaller $d_m$ may suffice.

We train a 4-layer Transformer model to perform 3-step reasoning. In this experiment, the sequence length is set to 21, the batch size is 1000, and the learning rate is $5\times 10^{-5}$. The model is trained for 500 epochs with a hidden dimension of $d_{m}=128$.

\begin{figure}[tbp]
    \centering 
    \includegraphics[width=0.8\textwidth]{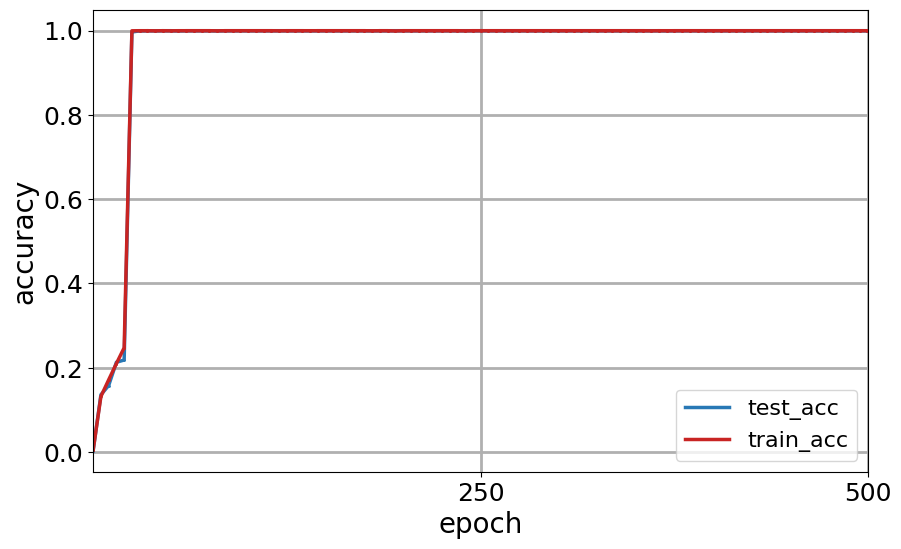} 
    \caption{L=4 $d_{m}=128$ } 
    \label{fig:example} 
\end{figure}

The experimental results indicate that the blue and red strings both rapidly approach 100\% success rates. Under the string reasoning condition, a transformer model with 128 hidden dimensions demonstrates the capability to effectively perform 3-step reasoning tasks (figure \ref{fig:example}). In contrast, under the parallel reasoning condition, an architecturally equivalent model with the same number of hidden dimensions achieves only an 11.9\% success rate.

\subsection{Computational Resources}

The experiments were conducted on a server with the following configuration:

\begin{itemize}
    \item 64 AMD EPYC 7742 64-Core Processor.
    \item 256GB of total system memory.
    \item 4 A100 GPUs with 40GB of memory.
\end{itemize}
 
The task can be completed within 3 weeks with 2 NVIDIA A100 GPU.

%% file: shared_ffn.tex
\subsection*{\textbf{One FNN shared across positions}}\label{app:shared_position_ffn}
The position index on the local decoding functions does not require different network parameters at different positions. The receiver's position can be identified from its attention-residual input. We establish this for the ideal states first, and then choose disjoint neighborhoods on which the local functions can be joined continuously.

\begin{proposition}\label{prop:shared_position_ffn}
Fix a layer and the finite family of admissible length-$n$ inputs over the chosen dictionary. The compact neighborhoods of the ideal FNN inputs at different positions can be chosen pairwise disjoint. The continuous local pre-normalization maps constructed in Step 2 can then be approximated, after the FNN residual and RMSNorm, by one position-wise FNN shared by all positions in that layer.
\end{proposition}
\begin{proof}
Put $B=3^L$ and $h=(B-1)/2$. For $i\ge2$, all coordinates that can occur in the ideal sparse code at position $i$ belong to
\begin{equation}\label{eq:shared_coordinate_blocks}
 \begin{aligned}
 \mathcal B_i&=\bigcup_{t=1}^{d}
 \{k_t-iB-r:r\in\mathbb Z,\ |r|\le h\},\quad 2\le i\le n,\\
 \mathcal B_1&=\{k_t-(n+1)B:1\le t\le d\}.
 \end{aligned}
\end{equation}
Here the shift within a chain is its token rank minus the rank of the original token, as in \eqref{output}. The information-quantity bound gives a maximum rank difference of $3^{l-1}$ at hidden layer $l\le L$, which is at most $h$. The full-dictionary spacing in \eqref{embedding_requirement} prevents collisions between different token embeddings and prevents wrap-around. For the same token, regular position-block centers are $B$ apart, while $2h=B-1<B$; the special block $\mathcal B_1$ is also separated from them. Thus $\mathcal B_i\cap\mathcal B_j=\varnothing$ whenever $i\ne j$, uniformly over all token choices and permutations.

For $1\le l<L$, write $w_{ij}^{(l)}$ for the causal softmax weights. Since $W^{vo(l)}=I$, the ideal input to the next FNN is
\begin{equation}\label{eq:shared_residual_marker}
 Y_i=\mX_i^{ao(l)}
 =(1+w_{ii}^{(l)})\mX_i^{(l)}
 +\sum_{j<i}w_{ij}^{(l)}\mX_j^{(l)}.
\end{equation}
The ideal codes have coefficients zero or one and nonempty, disjoint supports. At every coordinate of the receiver's own support, $Y_i$ therefore equals $1+w_{ii}^{(l)}>1$. At every other coordinate it is either zero or a weight $w_{ij}^{(l)}<1$. For $i=1$, $w_{11}^{(l)}=1$ and $Y_1=2\mX_1^{(l)}$, so the same conclusion holds. Hence
\begin{equation}\label{eq:shared_large_coordinates}
 \varnothing\ne\{a:(Y_i)_a>1\}
 =\operatorname{supp}(\mX_i^{(l)})\subseteq\mathcal B_i.
\end{equation}
For inputs from different positions, these sets of coordinates lie in disjoint blocks. In particular, the input vectors cannot be equal, even when they arise from different sequences.

The first layer has a separate position marker. With $\vp_i=e_i$, its residual contributes $1$ at positional coordinate $i$. The shifted positional contributions from indices $j\le i$ occupy coordinates $j-1<i$, except that $e_1R=e_{d_m}$. Token coordinates do not enter the first $n$ coordinates. Thus $Y_i^{0}$ is zero at positional coordinates greater than $i$, and equals $1$ at coordinate $i$. If $i>j$, $Y_i^{0}$ and $Y_j^{0}$ differ by $1$ at coordinate $i$. The first-layer inputs are therefore separated as well; no claim of a coefficient greater than $1$ is needed for that layer.

Let $\mathcal S_i^l$ denote the finite set of ideal attention-residual inputs at position $i$. The preceding separation implies
\begin{equation}\label{eq:shared_position_distance}
 \gamma_l=\min_{i\ne j}\ \min_{x\in\mathcal S_i^l,\,y\in\mathcal S_j^l}
 \|x-y\|_\infty>0.
\end{equation}
Choose a common radius $0<\rho_l<\min\{\gamma_l/3,\min_i\delta_i^{l+1}\}$ and replace each $\delta_i^{l+1}$ by $\rho_l$ in \eqref{eq:local_decoder_neighborhoods} and \eqref{Extended_f}. The resulting compact sets $K_i^l$ are pairwise disjoint, since their distance is at least $\gamma_l-2\rho_l>0$. Within each position, different target classes also remain separated because the radii have only decreased. All subsequent uses of $\delta_i^{l+1}$ refer to these reduced radii. They are chosen from the fixed ideal inputs before any approximation error is selected. Separation is used on neighborhoods, so the implemented network need not perform a discontinuous threshold test at $1$.

Write $g_i^l=\tilde f_i^l|_{K_i^l}$ for the continuous extension of the local map in Step 1 to these reduced neighborhoods. Its normalized output equals the required target throughout each class neighborhood. On $K^l=\bigcup_iK_i^l$, define
\begin{equation}\label{eq:shared_piecewise_ffn}
 F^l(y)=g_i^l(y)-y\qquad(y\in K_i^l).
\end{equation}
Because the finitely many compact sets are disjoint, this is a single well-defined continuous function on $K^l$. The subtraction of $y$ accounts for the FNN residual: $R_N(F^l(y)+y)=R_N(g_i^l(y))$. Extend each coordinate of $F^l$ continuously from the closed set $K^l$ to $\mathbb R^{d_m}$. By universal approximation, with a fixed continuous sigmoidal activation and sufficiently many hidden units, one FNN $\widehat F^l$ approximates this extension uniformly on $K^l$. Continuity of RMSNorm with $\epsilon>0$, uniformly on a compact neighborhood of the relevant pre-normalization vectors, then gives, for every $\eta>0$,
\begin{equation}\label{eq:shared_normalized_approximation}
 \max_i\ \sup_{y\in K_i^l}
 \bigl\|R_N(\widehat F^l(y)+y)-R_N(g_i^l(y))\bigr\|_\infty<\eta.
\end{equation}
Thus one set of FNN parameters serves all positions. Different layers may still have different FNN parameters, as in the Transformer architecture.
\end{proof}

%% file: appendix_extra.tex
\section{Additional Examples and Matching Simplification}\label{app:additional_examples}

\subsection{Permuted reasoning sequences and reasoning starts}\label{app:reasoning_examples}
The following remark and examples supplement the definitions in Section~\ref{sec:symbolic_reasoning_task}.

\begin{remark}
    In the definition of reasoning sequence we use the permutation to change the order of reasoning pairs which does not break the relation inside each reasoning pair. Moreover, no permutation should be applied to the original sequence $(x_i)$. For example, the sequence $(1,2,2,3,3,4)$ can be $(2,3,3,4,1,2)$ or $(3,4,2,3,1,2)$ under some certain permutations. Both of these sequences are related to the reasoning chain $((1,2),(2,3),(3,4))$. However,  it cannot be transformed into $(1,3,3,4,2,3)$ through any permutation that acts on reasoning pairs. 
\end{remark}

\begin{example}
  The sequence $(x_i)_{i\geqslant 1} = (\lfloor \frac{i}{2} \rfloor)_{i \geqslant 1}$ can be seen as a sorted reasoning sequence with constructing permutation $\sigma= \operatorname{Id}$ and constructing reasoning chain $((0,1),(1,2),(2,3),(3,4),\cdots)$. 
\end{example}

\begin{example}\label{nontrivial_example}
   The sequence $(x_i) = (1,2,6,3,2,4,3,5,4,6)$ has constructing reasoning chain
   \begin{equation*}
     (\va_m) = ((1,2),(2,4),(4,6),(6,3),(3,5))
   \end{equation*}
   and constructing permutation $\sigma$ satisfying
   \begin{equation*}
     \begin{aligned}
       \sigma(1)&=1,\quad \sigma(2)=4,\quad \sigma(3)=2,\\
       \sigma(4)&=5,\quad \sigma(5)=3.
     \end{aligned}
   \end{equation*}
   In this example,
   In this example,
   \begin{equation*}
     \begin{aligned}
       \va_1 &= (1,2) = (x_1,x_2), \\
       \va_2 &= (2,4) = (x_5,x_6), \\
       \va_3 &= (4,6) = (x_9,x_{10}).
     \end{aligned}
   \end{equation*}
  \begin{figure}[ht]
  \centering
    \begin{tikzpicture}[
        node distance=2cm,  
        box/.style={
            draw,   
            rounded rectangle,    
            minimum width=1.2cm,   
            minimum height=0.8cm,  
            thick,   
            font=\large,  
            align=center,
            scale=0.75
        }
    ]
        \node[box] (1) {$1$};
        \node[box] (2) [right of=1] {$2$};
        \node[box] (3) [right of=2] {$4$};
        \node[box] (4) [right of=3] {$6$};
        \node[box] (5) [right of=4] {$3$};
        \node[box] (6) [right of=5] {$5$};
        \draw[->,   thick,   >=Stealth] (1) -- (2);
        \draw[->,   thick,   >=Stealth] (2) -- (3);
        \draw[->,   thick,   >=Stealth] (3) -- (4);
        \draw[->,   thick,   >=Stealth] (4) -- (5);
        \draw[->,   thick,   >=Stealth] (5) -- (6);
    \end{tikzpicture}
    \caption{Reasoning task represents by $(x_i)$ in example \ref{nontrivial_example}.}\label{fig:permuted_chain}
    \end{figure}
\end{example}

\begin{example}\label{ex:permuted_reasoning_start}
     We set the sequence $(x_i) = (1,2,6,3, 2,4,3,5,4,6,4)$ with constructing reasoning chain $(\va_m)$ and constructing permutation $\sigma$ as in  example \ref{nontrivial_example}. In addition, the reasoning start is set to be $4$. Then one step reasoning result is $6$ and two steps reasoning result is $3$.
\end{example}

\clearpage
\subsection{Simplification preserving forward reasoning information}\label{app:matching_order}
We use the reasoning sequence $(x_i)_{1\leqslant i\leqslant 2s}$, constructing chain $(\va_m)_{1\leqslant m\leqslant s}$, permutation $\sigma$, and reasoning start $x_{2s+1}=\va^1_{m_0}$ from Section~\ref{sec:symbolic_reasoning_task}. The chain has no repeated vertices, and each reasoning pair occurs once. We retain $N_i^l=\{V_i^l,I_i^l\}$ for the simplified rules and write $\widehat N_i^l=\{\widehat V_i^l,\widehat I_i^l\}$ for the full process. In both cases, $l$ counts completed attention blocks.

In every block $l\geqslant0$, the full process allows a causal transmission $j\to i$, $j<i$, whenever at least one of the following holds:
\begin{equation}\label{eq:full_matching_conditions}
\begin{aligned}
 &\widehat V_j^l\cap\widehat V_i^l\ne\emptyset
 &&\text{(value matching)},\\
 &\widehat I_j^l\cap\widehat I_i^l\ne\emptyset
 &&\text{(same-position matching)},\\
 &2r-1\in\widehat I_j^l,\quad 2r\in\widehat I_i^l
   \quad\text{for some }1\leqslant r\leqslant s
 &&\text{(odd--even matching)}.
\end{aligned}
\end{equation}
All conditions use the input states; all eligible messages and the residual are united in parallel. Thus odd--even matching may use positions carried from earlier blocks. The simplified process uses only physical odd--even pairing in block $0$ and value matching thereafter.

\begin{proposition}\label{prop:forward_matching_equivalence}
Let $\mathcal V_+(m_0)=\{\va_m^2:m_0\leqslant m\leqslant s\}$ be the values strictly after the reasoning start along the constructing chain. For every $l\geqslant0$,
\begin{equation}\label{eq:forward_matching_equivalence}
 \widehat V_{2s+1}^l\cap\mathcal V_+(m_0)
 =V_{2s+1}^l\cap\mathcal V_+(m_0).
\end{equation}
Consequently, allowing all three matching mechanisms in every block preserves the last token's information available for forward reasoning and the maximal forward reasoning depth.
\end{proposition}

\begin{proof}
All intervals and comparisons between values below refer to their order along $(\va_m)$, not their numerical token labels.

\paragraph{1. Intervals and inclusion.}
Same-position matching implies value matching, since a common index $r$ supplies the common value $x_r$. Also, $V_i^l\subseteq\widehat V_i^l$: every simplified transmission remains eligible in the full process. Both value sets are intervals along the chain. This follows by induction: value matching joins intersecting intervals, while odd--even matching joins intervals containing the consecutive values $x_{2r-1}$ and $x_{2r}$. It therefore suffices to prove equality of the query intervals' right endpoints.

\paragraph{2. An auxiliary interval.}
For a fact position $i\leqslant2s$, define
\begin{equation}\label{eq:completed_matching_interval}
 \mathcal E_i^l=\widehat I_i^l\cap\{2,4,\ldots,2s\},
 \qquad
 U_i^l=\widehat V_i^l\cup\{x_{r-1}:r\in\mathcal E_i^l\}.
\end{equation}
This completes every carried even index $r$ to its reasoning pair
$(x_{r-1},x_r)=\va_{\sigma(r/2)}$; it is only an auxiliary construction.
Regard these completed pairs as edges and the remaining carried values as vertices. The resulting graph is connected: it starts as one edge or one vertex, and any eligible transmission joins two such graphs at a common vertex. For odd--even matching, that vertex is $x_{2r-1}$, already present in the receiver's completed pair. Hence $U_i^l$ is an interval, and $\mathcal E_i^l$ includes the even index of every edge inside it. If $\mathcal E_i^l$ is empty, $U_i^l$ is a singleton.

For an odd fact position $i$, the right endpoint of $U_i^l$ is always $x_i$. Indeed, the unique outgoing pair of $x_i$ has even index $i+1$, outside its causal prefix; a connected interval containing $x_i$ cannot cross that missing edge.

\paragraph{3. Coverage by a simplified even position.}
We claim that, for $l\geqslant1$ and $\mathcal E_i^l\ne\emptyset$,
\begin{equation}\label{eq:even_position_cover}
 r_i^l:=\max\mathcal E_i^l,
 \qquad U_i^l\subseteq V_{r_i^l}^l.
\end{equation}
At $l=1$, an even position carries only its own even index, and its pair has already been assembled by the simplified process. An odd position can carry only the even index of its value's unique incoming pair, if that index is earlier. Completing it gives the pair already assembled at that even position. Thus the claim holds.

Assume it holds at depth $l$. Every full-process sender's auxiliary interval intersects the receiver's auxiliary interval. A singleton sender therefore adds nothing to that interval. Each other sender $j$ is covered by the simplified node at $r_j^l\leqslant j$.

If the receiver $i$ is even, then $r_i^l=i$. Each sender's covering node precedes $i$ and intersects $V_i^l$, so the simplified process can transmit its entire covering set to $i$ in this block. This proves \eqref{eq:even_position_cover} at depth $l+1$ for even $i$.

For odd $i$, write $U_i^l=[u,x_i]$ and $U_i^{l+1}=[v,x_i]$, with intervals taken along the chain. If $\mathcal E_i^{l+1}$ is empty there is nothing to prove; otherwise set $r=r_i^{l+1}$.
If $r\in\mathcal E_i^l$, then $r=r_i^l$. Every sender's covering position is at most $r$, and its covering set intersects $V_r^l$ through $U_i^l$. A covering set at $r$ is retained; each earlier one can be transmitted to $r$ in this block. Their union covers $U_i^{l+1}$.

If $r\notin\mathcal E_i^l$, its pair lies to the left of $u$: all edges inside $[u,x_i]$ are already present, and the new interval cannot extend past $x_i$. Choose a sender whose auxiliary interval reaches the new left endpoint $v$. Its interval also reaches $u$, because it intersects $U_i^l$, and therefore contains the new pair at $r$. Its largest carried even index is exactly $r$, so $V_r^l$ covers $[v,u]$ by induction. If $\mathcal E_i^l\ne\emptyset$, the old receiver is covered at $r_i^l<r$; that covering set contains $[u,x_i]$ and can transmit to $r$ through the common value $u$. Otherwise $u=x_i$ and no further transmission is needed. In either case $V_r^{l+1}$ covers $[v,x_i]$, completing the induction.

For even fact positions, \eqref{eq:even_position_cover} and inclusion give
\begin{equation}\label{eq:even_matching_equality}
 V_i^l=\widehat V_i^l=U_i^l
 \qquad(i\leqslant2s\text{ even},\ l\geqslant1).
\end{equation}
For odd fact positions, every nonsingleton full value set is covered at a strictly earlier simplified even position, at the same depth.

\paragraph{4. The query's right endpoint.}
At depths $0$ and $1$, both query value sets are $\{x_{2s+1}\}$. Suppose their right endpoints agree at depth $l\geqslant1$, and denote this value by $b$. Consider any value $y$ strictly after $b$ that the full query receives in the next block, from position $j$.
If the transmission uses value matching, the sender's interval intersects the query interval and extends to $y$, so it contains $b$. If it uses odd--even matching, the sender contains $x_{2r-1}$ while the query contains its successor $x_{2r}$; the sender's interval extending to $y$ must again pass through $b$.

Thus $\widehat V_j^l$ contains both $b$ and $y$. By \eqref{eq:even_matching_equality} for even $j$, or \eqref{eq:even_position_cover} for odd $j$, a simplified even position preceding the query contains both values at depth $l$. Since $b\in V_{2s+1}^l$, the simplified query can receive $y$ from that position in the same block. Inclusion gives the reverse comparison, so the right endpoints remain equal. Induction and the interval property prove \eqref{eq:forward_matching_equivalence}.
\end{proof}

The equality concerns successor information at the last token. Total value sets can differ: for $(b,c),(a,b),(c,d)$ followed by query $c$, after two blocks the simplified query contains $\{b,c,d\}$ and the full query contains $\{a,b,c,d\}$, while both successor sets are $\{d\}$. Thus $C_i^l=|V_i^l|$ retains its original meaning in the simplified-process proofs; the equality above transfers the forward reasoning bounds to the full process without asserting equality of all total information quantities.